\documentclass[10pt,twocolumn,letterpaper]{article}

\usepackage{iccv}
\usepackage{times}
\usepackage{epsfig}
\usepackage{graphicx}
\usepackage{amsmath}
\usepackage{amssymb}
\usepackage{multirow}
\usepackage{algorithm}
\usepackage{algorithmic}
\usepackage{amsthm}
\usepackage{caption}
\usepackage{subcaption}
\usepackage{grffile}

\newcommand{\ceil}[1]{\bigg \lceil #1 \bigg \rceil}
\newtheorem{theorem}{Theorem}

\usepackage[pagebackref=true,breaklinks=true,letterpaper=true,colorlinks,bookmarks=false]{hyperref}

\iccvfinalcopy % *** Uncomment this line for the final submission

\def\iccvPaperID{1821} % *** Enter the ICCV Paper ID here

\ificcvfinal\fi
\begin{document}
	
	%%%%%%%%% TITLE
	\title{Smaller Models, Better Generalization}
	
	\author{Mayank Sharma\\
		Indian Institute of Technology, Delhi\\		
		{\tt\small eez142368@iitd.ac.in}
		\and
		Suraj Tripathi\\
		Indian Institute of Technology, Delhi\\		
		{\tt\small surajtripathi93@gmail.com}
		\and
		Abhimanyu Dubey\\
		MIT\\		
		{\tt\small abhimanyu1401@gmail.com}
		\and
		Jayadeva\\
		Indian Institute of Technology, Delhi\\		
		{\tt\small jayadeva@ee.iitd.ac.in}
		\and
		Sai Guruju\\
		Indian Institute of Technology, Delhi\\		
		{\tt\small saiguruju27@gmail.com}
		\and
		Nihal Goalla\\
	    Indian Institute of Technology, Delhi\\		
		{\tt\small goallanihal1995@gmail.com}
	}
	
	\maketitle

\begin{abstract} 
Reducing the network complexity has been a major research focus in recent years with the advent of mobile technology. Convolutional Neural Networks that perform various vision tasks without memory overhaul are the need of the hour. This paper focuses on qualitative and quantitative analysis of reducing the network complexity using an upper bound on Vapnik-Chervonenkis dimension, pruning and quantization. We observe a general trend in improvement of accuracies as we quantize the models. We propose a novel loss function that helps in achieving considerable sparsity at comparable accuracies to that of dense models. We compare various regularizations prevalent in the literature and show the superiority of our method in achieving sparser models that generalize well.
\end{abstract} 

\section{Introduction}
Deep Neural Networks have been very successful in variegation of tasks. They have been applied to Image classification \cite{krizhevsky2012imagenet,he2015delving,simonyan2014very}, Text analytics \cite{pennington2014glove,huang2012improving}, Handwriting generation \cite{graves2013generating}, Image Captioning \cite{karpathy2015deep}, Automatic Game playing \cite{mnih2013playing,silver2016mastering}, Speech Recognition \cite{hannun2014deep}, Machine translation \cite{bahdanau2014neural,sutskever2014sequence} and many others. Bengio et~al. \cite{lecun2015deep} and Schmidhuber \cite{schmidhuber2015deep} provides an extensive review of deep learning and its applications. \\
The representational power of a neural network increases with its depth as is evident from the architectures like Highway Networks \cite{srivastava2015training} (32 layers and 1.25M parameters) and ResNet \cite{he2016deep} (110 layers has 1.7M parameters). Such large number of weights presents a challenge in terms of storage capacity, memory bandwidth and representational redundancy. For example, widely used models like AlexNet Caffemodel is over 200MB, and the VGG-16 Caffemodel is over 500MB.
With advent of mobile technologies and IoT devices the need for faster and accurate computing has arisen. Sparse matrix multiplications and convolutions are a lot faster than their dense counterparts. Furthermore, a sparse model with few parameters gain advantage in terms of better generalization ability thereby preventing overfitting. Effect of various regularizers ($L_0, L_1 and L_2$) on CNN (Convolutional Neural Networks) are studied in \cite{collins2014memory}.\\
In this paper we introduce a novel loss function to achieve sparsity by minimizing a convex upper bound on Vapnik-Chervonenkis (VC) dimension.
We first derive an upper bound on the VC dimension of the classifier layer of a neural network, and then apply this bound on the intermediate layers in the neural networks, in conjunction with the weight-decay ($L_2$ and $L_1$ norms) regularization bound. This result provides us with a novel error functional to optimize over with backpropagation for training neural network architectures, modified from the traditional learning rules.\\

This learning rule adapts the model weights to minimize both empirical error on training data as well as the VC dimesion of the neural network. With the inclusion of a term minimizing the VC dimension, we aim to achieve sparser neural networks, which allow us to remove a large number of synapses and neurons without any penalty on empirical performance.\\

Finally, we demonstrate the consistent effectiveness of the learning rule across a variety of learning algorithms on various datasets across learning task domains. We see that the data dependent rule promotes higher test set accuracies, faster convergence and achieves smaller models across various architectures such as Feedforward (Fully Connected) Neural Networks (FNNs) and Convolutional Neural Networks (CNNs) confirming our hypothesis that the algorithm indeed controls model complexity, while improving generalization performance.\\

The rest of the paper is organised as follows  - in Section \ref{sec:relatedwork} we provide a brief overview of the recent relevant work in complexity control and generalization in deep neural networks, and in Section \ref{sec:vcbound} we provide the derivation for our learning rule, and proof for theoretical bounds. Section \ref{sec:quantization} describes the effect of quantization on VC dimension bound of the network. In the subsequent section \ref{sec:results} we describe our experimental setup and methodology, along with qualitative and quantitative analyses of our experiments.
\section{Related Works}\label{sec:relatedwork}
Compression of deep nets have been widely studied. Network pruning and quantization are the methods of choice. Researchers have used weights and neuronal removal to instigate sparsity. \cite{han2015learning} used iterative deletion of weights and neurons to achieve sparsity. \cite{zhou2016less,scardapane2016group} used group sparse regularization on weights to incorporate sparsity. \cite{sun2016sparsifying} used iterative sparsification based on neural correlations. \cite{liu2014pruning} used optimal brain damage to enforce sparsity. \cite{srinivas2015data} removed redundant neurons based on saliancy of two weight sets. \cite{wolfe2017incredible} used second order Taylor information to prune neurons.\cite{aghasi2016net} pruned the net using sparse matrix transformation keeping the layer input and output close to the original unpruned model.\cite{srinivas2016training} used a bimodal regularizer to enforce sparsity and \cite{babaeizadeh2016noiseout} merged two neurons with high correlations.\\
A rich body of literature exist on quantizing the models as well. \cite{han2015deep} build their model on top of their earlier model, by adding quantization and Huffman coding. \cite{lin2015neural} used weight binarization and quantizing the learned representations in each layers to achieve the same. In their work \cite{rastegari2016xnor} binarized both weights and inputs to the convolutional layers. \cite{mellempudi2017mixed} proposed cluster based quantization method to convert pre-trained full precision weights to ternary weights with minimal loss in accuracy. \cite{hubara2016quantized} quantized weights, activations and incorporated quantized gradients with 6 bits in their training.
\section{Sparsifying Neural Networks through Pruning}\label{sec:vcbound}
In this section we derive an upper bound on the VC dimension $\gamma$. This proof is an extension of the one in \cite{jayadeva2015learning}. Vapnik \cite{vapnik98} showed that the VC dimension $\gamma$ for fat margin hyperplane classifiers with margin $d \geq d_{min}$ satisfies
\begin{equation}\label{eqnh}
\gamma \leq 1 + \operatorname{Min} \bigg( \ceil{\frac{R^2}{d_{min}^2}}, n \bigg)
\end{equation}
Let us consider a dataset $X \in \Re^{M \times n}$ with $M$ samples and $n$ features. The individual samples are denoted by $x^i \in \Re^n$.
where $R$ denotes the radius of the smallest sphere enclosing all the training samples. We first consider the case of a linearly separable dataset. By definition, there exists a hyperplane $w^Tx + b = 0$, parameterized by $w \in \Re^n$ and a bias term $b$ with positive margin $d$ that can classify these points with zero error. We can always choose a value $d_{min} < d$; for all further discussion we assume that this is the case. The samples are assumed to be in a high dimension; this assumption is reasonable because the samples inherently have a large number of features and are thus linearly separable, owing to Cover's theorem \cite{cover1968capacity}, or they have been transformed from the input space to a high dimensional space by using a nonlinear transformation. The case when the samples are linearly separable and in a small dimension is not interesting as these are of a trivial nature.
Thus we have,
\begin{equation}\label{eqnh1}
\gamma \leq 1 + \frac{R^2}{d^2_{min}}
\end{equation}
Let us consider the problem of minimizing the fraction as minimizing the upper bound on VC dimension.
\begin{gather} \label{eqnh2}
\operatorname{Min} \frac{R^2}{d^2_{min}}
\end{gather}
Since, both the numerator and denominator are positive quantities with $d_{min} > 0$ and $d_{min} < d$, we can alternatively write (\ref{eqnh2}) as:
\begin{gather} \label{eqnh3}
\operatorname{Min} \frac{R}{d}
\end{gather}
We simplify the value of the fraction $\frac{R}{d}$, to attain a tractable convex bound in term of the weights of network.
\begin{gather}
\frac{R}{d} =  \bigg(\frac{ \operatorname*{max}_i \|x^i\| }{ \operatorname*{min}_{i} \frac{\|w^Tx^i +b\|}{\|w\|}}\bigg)\\
= \bigg(\frac{ \operatorname*{max}_i \|x^i\| \|w\| }{ \operatorname*{min}_{i} \|w^Tx^i +b\| } \bigg) \label{eqnh4}
\end{gather}
Without proper scaling of $w$ and $b$, we can write the minimum value of distance of correctly classified point to be $1$.
\begin{gather} \label{eqnh5}
\operatorname*{min}_{i} \|w^Tx^i +b\| =1 
\end{gather}
Using (\ref{eqnh5}), we convert (\ref{eqnh4}) to the following optimization problem.
\begin{gather}\label{eqnh6}
\frac{R}{d} = \big(\operatorname*{max}_i \|x^i\| \|w\| \big)
\end{gather}
Since, for two numbers $A$ and $B$, the following inequality holds:
\begin{gather} \label{eqnh7}
\|A\|^2 + \|B\|^2 \geq \|A\|\|B\|
\end{gather}
Applying the inequality (\ref{eqnh7}) to (\ref{eqnh6}), we achieve the following upper bound on the fraction 
\begin{gather} \label{eqnh8}
\frac{R}{d} \leq  \big(\operatorname*{max}_i \|x^i\|^2 + \|w\|^2 \big)
\end{gather}
For a separating hyperplane $w^Tx^i +b$ that passes through the data, the maximum distance of the point from the plane, is greater than the maximum radius of the data. Thus we can extend the bound on radius of dataset as:
\begin{gather}\label{eqnh9}
\operatorname*{max}_{i}\|x^i\| \leq \ \operatorname*{max}_{i}  \frac{\|w^Tx^i +b\|}{\|w\|}
\end{gather}
Using the bound derived in (\ref{eqnh9}), we can write (\ref{eqnh8}) as:
\begin{gather}\label{eqnh10}
 \frac{R}{d_{min}} \leq \bigg(\operatorname*{max}_i  \frac{\|w^Tx^i +b\|^2}{\|w\|^2} + \|w\|^2 \bigg)
\end{gather}
For positive numbers $a_i,\,\, i \in \{1, \ldots, N\}$, the following inequality holds,
\begin{gather}
\operatorname{Max} a_i \leq \sum_{i=1}^{N} a_i \label{eqnh11}
\end{gather}
Using (\ref{eqnh11}) in (\ref{eqnh10}), we have the following bound
\begin{gather}\label{eqnh12}
\frac{R}{d} \leq \bigg( \sum_{i=1}^{M} \frac{\|w^Tx^i +b\|^2}{\|w\|^2} + \|w\|^2 \bigg)
\end{gather}
Finally, we arrive at the convex and differentiable version of the bound on VC dimension, that can be minimized using stochastic gradient descent and can used in conjugation with various architectures. The following bound acts as a data dependent regularizer when used alongside the loss function minimization. Here we present the effectiveness of the bound for reducing the number of connections of the network. %This bound is similar to the one presented in %\cite{anonymous}, where the focus of the paper was to show generalization impovement with application of the VC bound. However, we present an alternative proof to the upper bound on VC dimension and show its advantage in terms of network sparsity achieved through pruning and quantization thereby containing the capacity of a deep neural network.

\begin{gather} \label{eqnh13}
\Gamma = \operatorname*{Min} \big( \sum_{i=1}^{M} \|w^Tx^i +b\|^2 + C \|w\|^2 \big)
\end{gather}

\subsection{A bound on Neural network}
We now use the bound (\ref{eqnh13}) in the context of a multi-layer feedforward neural network. Consider a neural with multiple hidden layers for the problem of multiclass classification with $K$ classes. Let the number of neurons in the penultimate layer be denoted by $l$, and let their outputs be denoted by $z_1, z_2, \ldots, z_l$; let the corresponding connecting weights for the classifier layer be denoted by $w_{c_i} \in \Re^K,\,\, \forall \,\, i \in \{1, \ldots, l\}$ respectively. One may view the outputs of this layer as a map from the input $x$ to $\phi(x)$, i.e. $z = \phi(x)$. The biases of at the output are denoted by $b_{c_i} \in \Re \,\, \forall \,\, i \in \{1, \ldots, K\}$. The $j^{th}$ score for $i^{th}$ input pattern at the output is given by $net_j^i = w_{c_j}^T z^i + b_{c_j}$.
For the purposes of this paper, we use multiclass hinge loss following the works of Tang et~al., \cite{tang2013deep}, where the authors state superiority of hinge loss over softmax loss.
Thus applying the bound (\ref{eqnh13}) on the classification layer of neural network, lead us to the following optimization problem:
\begin{gather} 
\operatorname{Min} E =  \sum_{i=1}^{M}\sum_{j\neq y_i}^{K} \max(0, 1 - net_{y_i}^i + net_j^i) + \nonumber\\
\frac{C}{2}\sum_{j=1}^{K}\|w_{c_j}\|_2^2 + \frac{D}{2} \sum_{i=1}^{M}\sum_{j=1}^{K} (net_j^i)^2 \label{eqnh14}
\end{gather}

\subsection{Application of the bound on hidden layers}
The great advantage with this bound is its ability to be applied to pre-activations in the net across all the layers. When applied to the pre-activations in a net, it is interpreted as a $L_2$ regularizer. It forces pre-activations to be close to zero. For ReLu activation functions $max(0,x)$, our data dependent regularizer forces the pre-activations for each layer to be close to zero. Thus, it in turn enforces sparsity at neuronal levels in the intermediate layers. In principle, during back-propagation this tantamount to solving a least squares problem for each neuron where the targets are all $0$. Consider a feedforward architecture with $P$ hidden layers. For an intermediate layer $h$, the let the activations of the layer $h-1$ with $l_{h-1}$ neurons be $z_{h-1} \in \Re^{l_{h-1}}$. Let $w_{h_i} \in \Re^{l_{h-1}},\,\, \forall \,\, i \in \{1, \ldots, l_h\}$ be the weights of the layer $h$ going from $h-1$ to $h$ and $b_{h_i}$ be the set of biases. Let us assume that the targets for each sample for each pre-activations $a_{h_i}\,\, \forall \,\, i \in \{1, \ldots, l_h\}$ is $0$. Hence, the application of (\ref{eqnh13}) on pre-activations with  ReLu activation function, is equivalent to the following minimization problem.
\begin{gather}\label{eqnh15}
\operatorname{Min} \frac{1}{2}\sum_{j=1}^{l_h}\|w_{h_j}\|_2^2 + \frac{D}{2} \sum_{i=1}^{M}\sum_{j=1}^{l_h} \left(0 - (w_{h_j}^T z_{h-1}^i+b_{h_j}) \right)^2
\end{gather}

With the application of VC bound (\ref{eqnh15}) to all the layers, the final minimization problem can be derived from (\ref{eqnh14}) as:
\begin{gather} \label{eqnh16}
\operatorname{Min} E = \frac{C}{2}\sum_{h=0}^{P-1}\sum_{j=1}^{l_h}\|w_{h_j}\|_2^2 + \nonumber \\ 
\frac{C}{2}\sum_{j=1}^{K}\|w_{c_j}\|_2^2 + \frac{D}{2} \sum_{i=1}^{M}\sum_{l=0}^{P-1}\sum_{j=1}^{l_h} (w_{h_j}^T z_{h-1}^i + b_{h_j})^2 + \nonumber \\
\frac{D}{2} \sum_{i=1}^{M}\sum_{j=1}^{K} (net_j^i)^2 + \sum_{i=1}^{M}\sum_{j\neq y_i}^{K} \max(0, 1 - net_{y_i}^i + net_j^i) 
\end{gather}
\section{Trade-off between margin and error: Role of quantization}\label{sec:quantization}
Consider a binary classification problem with $M$ samples where $i^{th}$ sample is denoted as $x^i \in \Re^n$ and its corresponding label is represented as $y_i \in \{-1,1\}$. Let us define a fat margin hyperplane classifier denoted by $\sum_{j=1}^{n} (w_jx_j) + b = 0$ where, $w_j \in \Re \,\, \forall \,\, j \in \{1,\ldots,n\}$ be the weights and $b \in \Re$ be the bias term. Let $w_j^Q$ be the quantized weights and $b^Q$ be the quantized bias term. Without loss of generality, we can consider hyperplanes passing through the origin. To see that this is possible, we augment the co-ordinates of all samples with an additional dimension or feature whose value is always $1$, i.e. the samples are given by $\hat{x}^i \leftarrow \{x^i; 1\}, i = 1, 2,\ldots, M$; also, we assume that the weight vector is $(n+1)$-dimensional, i.e. $u \leftarrow \{w; b\}$. Thus, the classifier then becomes $\sum_{j=1}^{n+1} (u_j\hat{x}_j) = 0$. Following the above notation, quantized version of vector $u$ is denoted as $u^Q$.
\begin{theorem}\label{th2}
	Consider full precision and a quantized fat margin classifiers with upper bounds on VC dimensions as $\Gamma$ and $\Gamma^Q$. If $(|u_j|-|u_j^Q|) \geq 0 \,\, \forall \,\, j \in \{1,\ldots,n\}$, then the quantized classifier has smaller VC bound ($\Gamma^Q < \Gamma$).
\end{theorem}
\begin{proof}
	Given a set of linearly separable data points and the two fat margin classifiers, former with full precision and latter with quantized set of weights. If the predicted label for each data point assigned by each individual classifiers is the same, which implies that the two classifiers have same accuracies, then the differences in the scores for each sample multiplied by its individual class should be positive. 
	%Subtracting  (\ref{fat_uQ}) from eq. (\ref{fat_Q}), we get,
	\begin{gather}
		y_i \sum_{j=1}^{n+1} (\Delta u_j\hat{x}_j^i) \geq 0 \,\, \forall \,\, i \in \{1,\ldots,M\} \implies \label{quant_delta}\\		
		\begin{cases}
			\sum_{j=1}^{n+1} (\Delta u_j\hat{x}_j^i) \geq 0 \,\, \forall \,\, i \in \{1,\ldots,M \,:\, y_i = 1\} \\
			\sum_{j=1}^{n+1} (\Delta u_j\hat{x}_j^i) \leq 0 \,\, \forall \,\, i \in \{1,\ldots,M \,:\, y_i = -1\}
		\end{cases}		
    \end{gather}
    where, $\Delta u_j = u_j- u_j^Q$.\\
    It can be easily shown that (\ref{quant_delta}) is true if, 
    \begin{gather}
	    (|u_j|-|u_j^Q|) \geq 0 \,\, \forall \,\, j \in \{1,\ldots,n\} \label{cond1}
    \end{gather}
    The condition (\ref{cond1}), translates to the fact that we assign smaller number mantissa bits to the weights or during reduction in fraction bits $|u_j^Q|$ is smaller than $|u_j|$. The argument for this condition comes from the fact that if (\ref{cond1}) holds then the sign of $\Delta u_j$ remains the same as that of $u_j$ or $u_j^Q$. Now since quantization does not allow flipping of signs of each individual bits, (\ref{cond1}) allows for the same sign of the sum given by eq. (\ref{quant_delta}).    
    This implies,
    \begin{gather}
	    \|u\|_2^2 \geq \|u^Q\|_2^2 \label{norm_inq}\\ 
	    \|u^Tx^i\|_2^2 \geq \|{u^Q}^Tx^i\|_2^2 \,\, \forall \,\, i \in \{1,\ldots,M\} \label{norm_uTx_inq}
    \end{gather}
    From, eq. (\ref{eqnh13}) where we define $\Gamma = C\|u\|_2^2 +  \|u^Tx^i\|_2^2 $, analogous to it, the quantized counterpart can be defined as $\Gamma^Q = C \|u^Q\|_2^2 +  \|{u^Q}^Tx^i\|_2^2 $. Now, using eq. (\ref{norm_inq}) and eq. (\ref{norm_uTx_inq}), we have,
    \begin{gather}
    \Gamma^Q \leq \Gamma
    \end{gather}
    Thus by introducing the quantization, one can reduce the complexity of the classifier. This is also evident from the fact that the size of hypothesis class $H$ reduces as the precision is reduced.
\end{proof}

\section{Empirical Analysis and Observations}\label{sec:results}
We determine the effectiveness of network pruning and quantization on various network architectures like Convolutional Neural networks (CNNs) and fully-connected neural (FNN) nets using various data independent regularizers such as L1 norm and L2 norm on weights and dropout and the proposed data dependent regularizer (\ref{eqnh13}).

\subsection{Setup and Notation}
All our experiments are run on a GPU cluster with NVIDIA Tesla K40 GPUs, and implementations were done using the assistance of the Caffe \cite{jia2014caffe} library for CNNs, while the experiments for FNNs were done using Tensorflow \cite{abadi2016tensorflow} and quantization of FNNs was implemented using Matlab \cite{guide1998mathworks}. \\
\textbf{Hyperparamter settings}: The two main hyperparameters in our experiments are $C$ and $D$, which are described in section \ref{sec:vcbound}. The two hyperparameters were tuned in the range of $[10^{-04},10^{-01}]$ and  $[10^{-08},10^{-04}]$ in the multiples of $10$. The other parameters such as dropout rate was kept at their default values for densely connected nets and quicknet. The learning rate was tuned for two values namely $10^{-02}$ and $10^{-01}$. For CNNs the learning rate was multiplied by $0.1$ after every $100000$ iterations, whereas for FNNs the learning rate was decreased as $\frac{1}{epoch}$ after every epoch (one complete pass of data). The total number of iterations was kept to be $230000$ for CNNs and $500$ epochs for FNNs.\\
The notation used for simplicity in understanding experimental results is given as,
\begin{table}[htbp]
	\centering
	\begin{tabular}{|c|c|}
		\hline
		Symbols & Meaning \\
		\hline
		H    & Hinge loss \\
		W2    & $L_2$ regularization  \\
		W1    & $L_1$ regularization  \\
		LCL  & LCNN applied only on last layer \\
		LCA  & LCNN applied on all layers \\
		D     & Dropout \\
		BN    & Batch normalization \\
		\hline
	\end{tabular}%
	\caption{Tabular representation of notation.}
	\label{tab:notations1}%
\end{table}%

\subsection{Network Pruning}
To analyse the efficacy of our regularizer in attaining sparsity we perform pruning of the network after training has finished. Firstly, we select a minimum weight threshold of $1e-03$. Then, we calculate the absolute value of weights in each layer, subsequently we divide the difference between maximum value of weights in each layer and the minimum threshold value into 50 (for FNNs) or 100 (for CNNs) steps. In the last step, we loop over these 50 steps and prune the weights whose absolute magnitude is below the step value.
\subsubsection{CNNs: Datasets}
Our first set of experiments are performed on image classification task using CNNs. Table (\ref{tab:datasets_CNN}) describes the standard image classification dataset used in the pruning and quantization experiments. 
\begin{table}[htbp]
	\centering
	\caption{Dataset used for CNN experiments}
	\scalebox{0.7}{
		\begin{tabular}{|l|r|r|r|r|r|}
			\hline
			name  & \multicolumn{1}{|l|}{features} & \multicolumn{1}{|l|}{classes} & \multicolumn{1}{|l|}{train size} & \multicolumn{1}{|l|}{val size} & \multicolumn{1}{|l|}{test size} \\
			\hline
			Cifar 10 \cite{krizhevsky2009learning}   & 32$\times$32$\times$3   &10     & 50000 & 5000  &5000 \\	
			
			\hline
		\end{tabular}%
	}
	\label{tab:datasets_CNN}%
\end{table}%

\subsubsection{CNNs: Experiments}
We studied the effect of pruning and quantization on two architectures of CNNs, namely Caffe quicknet \cite{jia2014caffe} and Caffe implementation of densly connected convolutional nets \cite{huang2016densely} with 40 layers. We study various regularization and found that data dependent regularization achieves maximum sparsity, thus maximum compression ratio when compared to its contemporary regularizations. \\

Table \ref{tab:c_10_5_comp_ratio} shows the compression ratio achieved when we prune the trained model. $L_2$ weight regularization achieves the best compression followed by our data dependent regularizer, whereas table \ref{tab:c_10_5_acc}  shows the accuracies, our data dependent regularizer reaches the best accuracy in the pool, keeping up the sparsity.
We compare the effect of pruning and quantization on various regularizers visually using 2 dimensional tSNE plots of the final layer of densely connected CNNs. Figure (\ref{fig:tSNE_visualization}) describes the results. Here we observe that data dependent regularizer allows forming of compact clusters thus achieving better generalizations for Cifar 10 dataset. The plots for pruned and quantized networks are visually similar, yet on closer inspection one finds that some of the clusters like Automobile, Horse, Cat and Airplane gets better clusters in terms of compactness and better separability than their unpruned and un-quantized counterparts.

% Table generated by Excel2LaTeX from sheet 'Sheet8'
\begin{table}[htbp]
	\centering
	\caption{Compression ratio for cifar 10 quick net model}
	\scalebox{0.7}{
	\begin{tabular}{|l|r|r|}
		\hline
		& \multicolumn{1}{|l|}{Pruning} & \multicolumn{1}{|l|}{Quantization} \\
		\hline
		S     & 1.41  & 1.28 \\
		S + LCA & 1.29  & 1.07 \\
		S + W & \textbf{6.95}  & \textbf{6.03} \\
		S + W + BN & 1.92  & 2.33 \\
		S + W + BN + LCA & 1.33  & 1.93 \\
		S + W + BN + LCA + D & 1.16  & 2.20 \\
		S + W + D & 3.20  & 1.46 \\
		S + W + D + BN & 1.56  & 2.48 \\
		S + W + D + LCA & 3.77  & 1.53 \\
		S + W + D + LCL & 2.65  & 1.05 \\
		S + W + LCA & 1.89  & 1.04 \\
		S + W + LCL & 3.95  & 1.08 \\
		\hline
	\end{tabular}%
}
	\label{tab:c_10_5_comp_ratio}%
\end{table}%

% Table generated by Excel2LaTeX from sheet 'Sheet9'
\begin{table}[htbp]
	\centering
	\caption{Accuracies for cifar 10 quick net model}
	\scalebox{0.7}{
	\begin{tabular}{|l|r|r|r|}
		\hline
		& \multicolumn{1}{|l|}{Original acc} & \multicolumn{1}{|l|}{Pruned acc} & \multicolumn{1}{|l|}{Quantization acc} \\
		\hline
		S     & 0.73  & 0.72  & 0.71 \\
		S + LCA & 0.77  & 0.77  & 0.75 \\
		S + W & 0.77  & 0.76  & 0.76 \\
		S + W + BN & \textbf{0.80 } & 0.79  & 0.77 \\
		S + W + BN + LCA & 0.78  & 0.77  & 0.73 \\
		S + W + BN + LCA + D & 0.79  & 0.78  & 0.78 \\
		S + W + D & 0.77  & 0.76  & 0.76 \\
		S + W + D + BN & 0.79  & 0.78  & 0.79 \\
		S + W + D + LCA & 0.74  & 0.73  & 0.79 \\
		S + W + D + LCL & 0.78  & 0.73  & 0.77 \\
		S + W + LCA & 0.77  & 0.76  & 0.78 \\
		S + W + LCL & \textbf{0.79}  & 0.78  & 0.79 \\
		\hline
	\end{tabular}%
}
	\label{tab:c_10_5_acc}%
\end{table}%

Table \ref{tab:cifar10_densenet} shows the accuracies of Cifar10 before and after pruning and quantization on densely connected CNNs \cite{huang2016densely}. We observe that our regularization performs equally well when used in conjugation with dropout and $L_2$ weight regularizer.
% Table generated by Excel2LaTeX from sheet 'Sheet10'
\begin{table}[htbp]
	\centering
	\caption{Accuracies for cifar 10 densely connected CNN }
	\scalebox{0.7}{
	\begin{tabular}{|l|r|r|r|}
		\hline
		& \multicolumn{1}{|l|}{Original acc} & \multicolumn{1}{|l|}{Pruned Acc} & \multicolumn{1}{|l|}{Quantization acc} \\
		\hline
		H +D  & 0.901 & 0.887 & 0.901 \\
		H +W +D & \textbf{0.924 }& 0.913 & 0.923 \\
		H +W +LCL +D &\textbf{ 0.924} & 0.914 & 0.920 \\
		H + W & 0.900 & 0.886 & 0.897 \\
		H + W + LCL & 0.895 & 0.888 & 0.889 \\
		H +W1 & 0.866 & 0.853 & 0.857 \\
		H +W1 +LCL & 0.857 & 0.847 & 0.854 \\
		\hline
	\end{tabular}%
}
	\label{tab:cifar10_densenet}%
\end{table}%

Figure \ref{fig:pruned_quant_acc} shows the accuracies of various algorithms with the total number of bits after we perform the first round of pruning. We see that our regularizer has the best set of accuracies (H +W + LCL) among all the algorithms and is quite robust to the changes in the total number of bits.
\begin{figure}[b]
	\centering
	\includegraphics[scale=0.2]{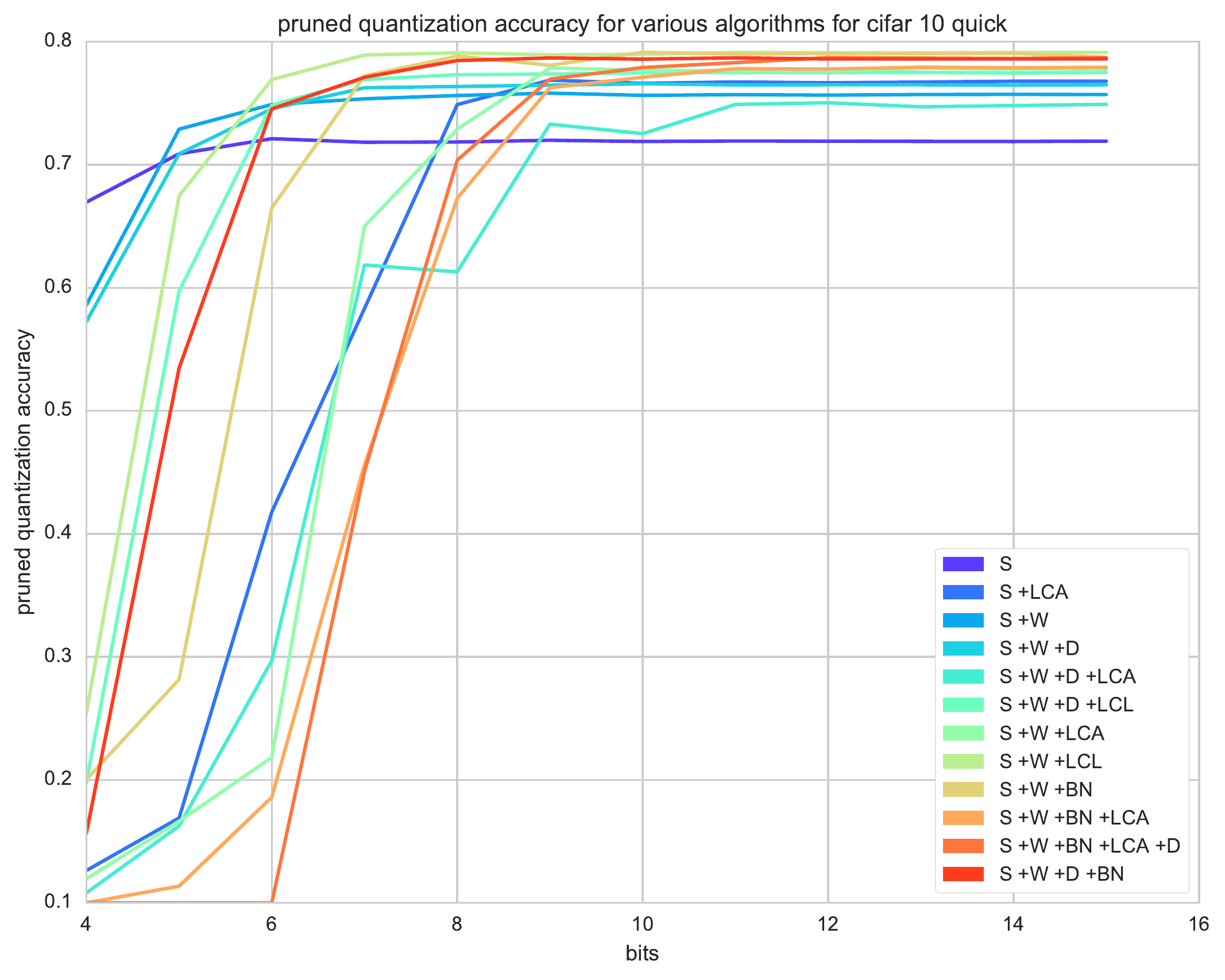}
	\caption{Accuracies for various algorithms after pruning and then quantizing the number of bits}
	\label{fig:pruned_quant_acc}
\end{figure}

\begin{figure*}[t]
	\centering
	\begin{subfigure}[b]{0.22\textwidth}
		\centering
		\includegraphics[scale=0.08]{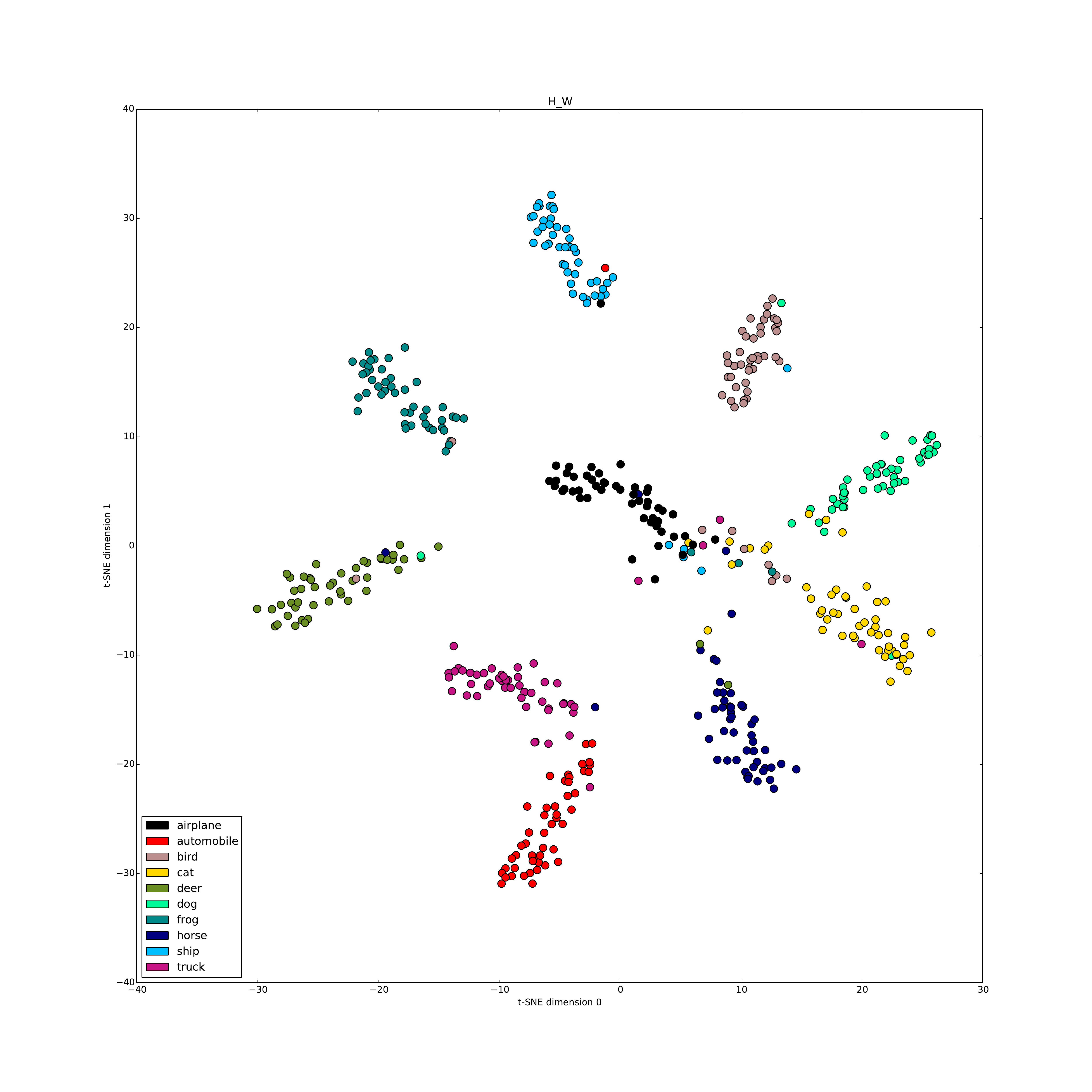}
		\caption{H + W2}
		\label{fig:tsne_c10_dense_H_W}
	\end{subfigure}
	\begin{subfigure}[b]{0.22\textwidth}
		\centering
		\includegraphics[scale=0.08]{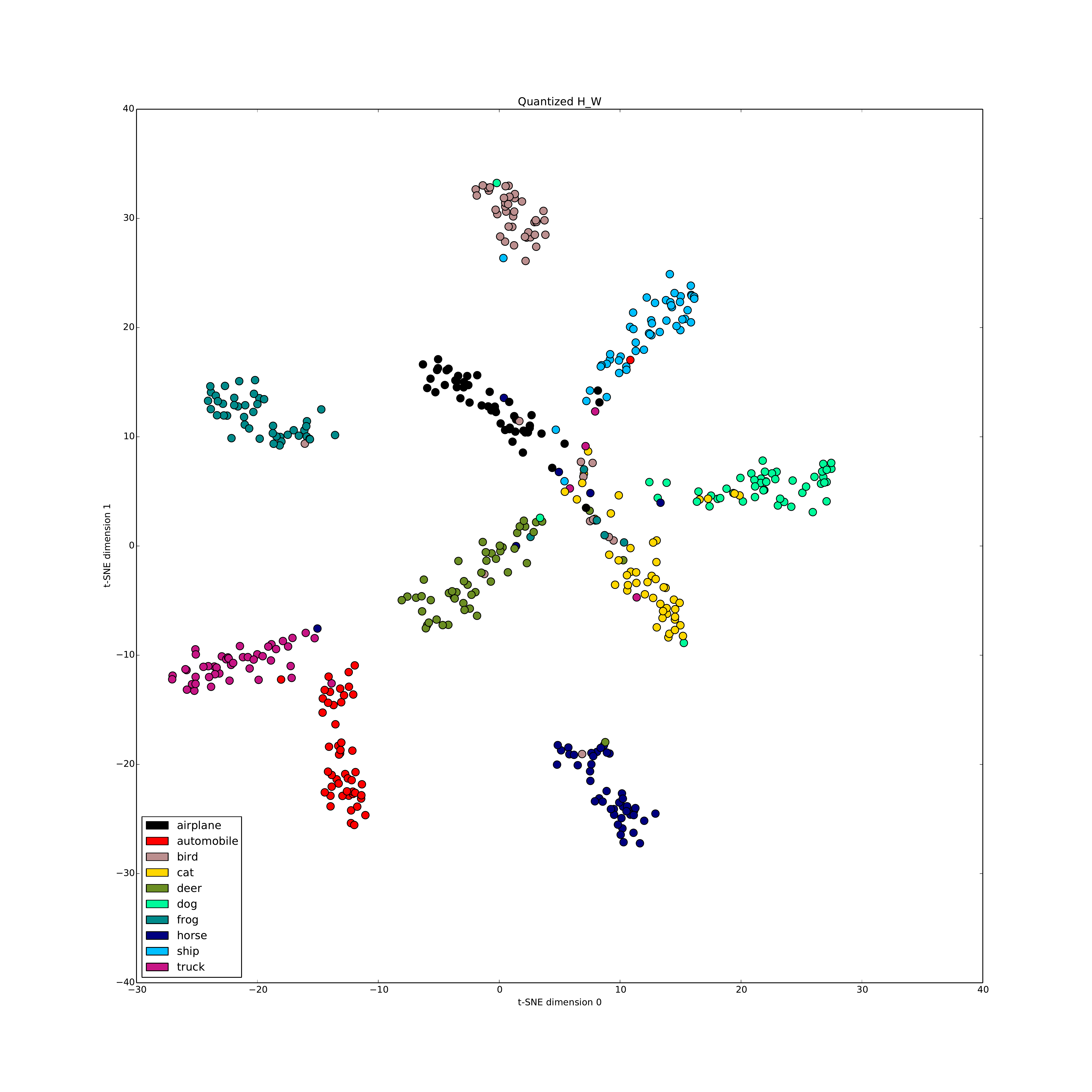}
		\caption{H + W2 + P + Q}
		\label{fig:tsne_c10_dense_H_W_Q}
	\end{subfigure}
	\begin{subfigure}[b]{0.22\textwidth}
		\centering
		\includegraphics[scale=0.08]{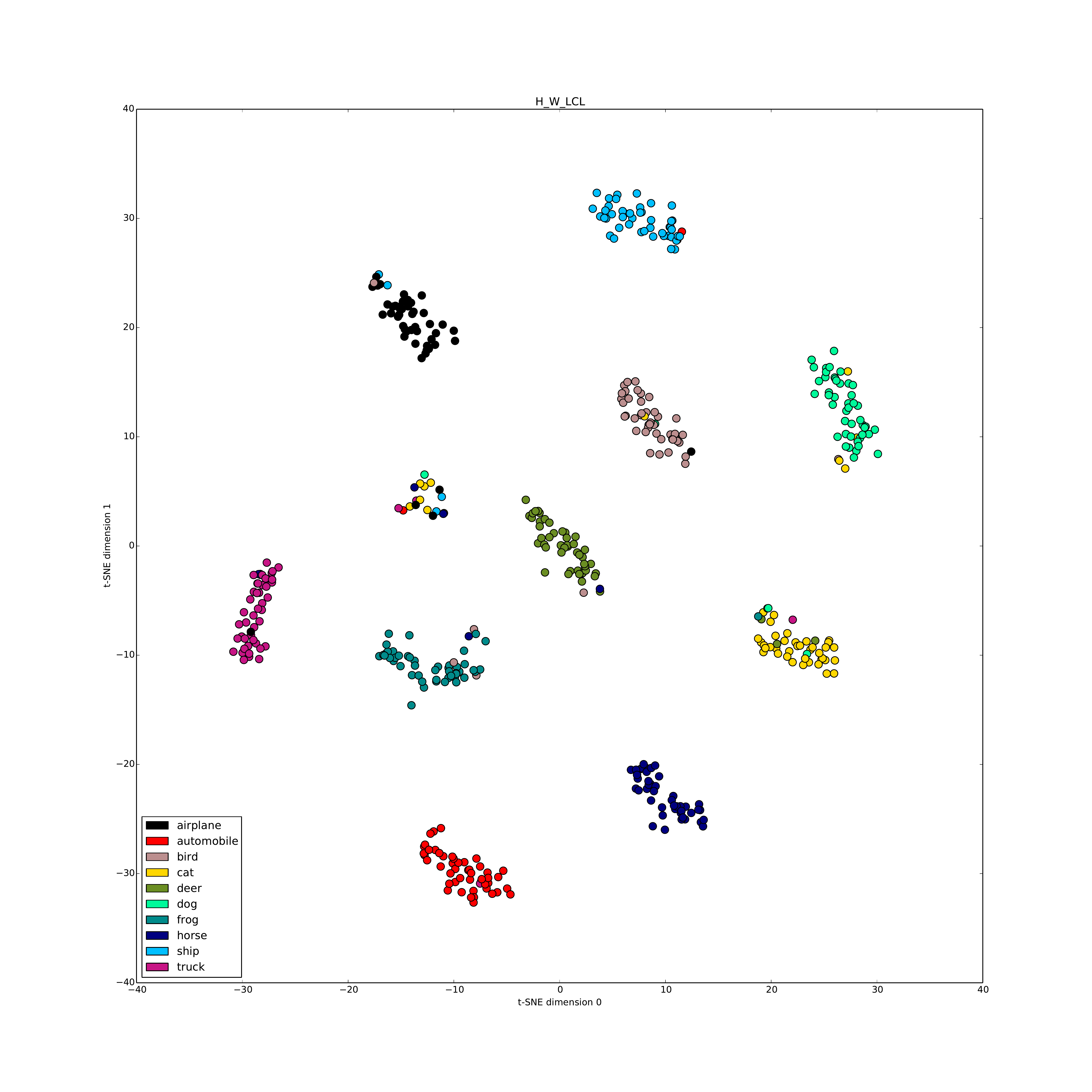}
		\caption{H + W2 + LCL}
		\label{fig:tsne_c10_dense_H_W_LCL}
	\end{subfigure}
	\begin{subfigure}[b]{0.22\textwidth}
		\centering
		\includegraphics[scale=0.08]{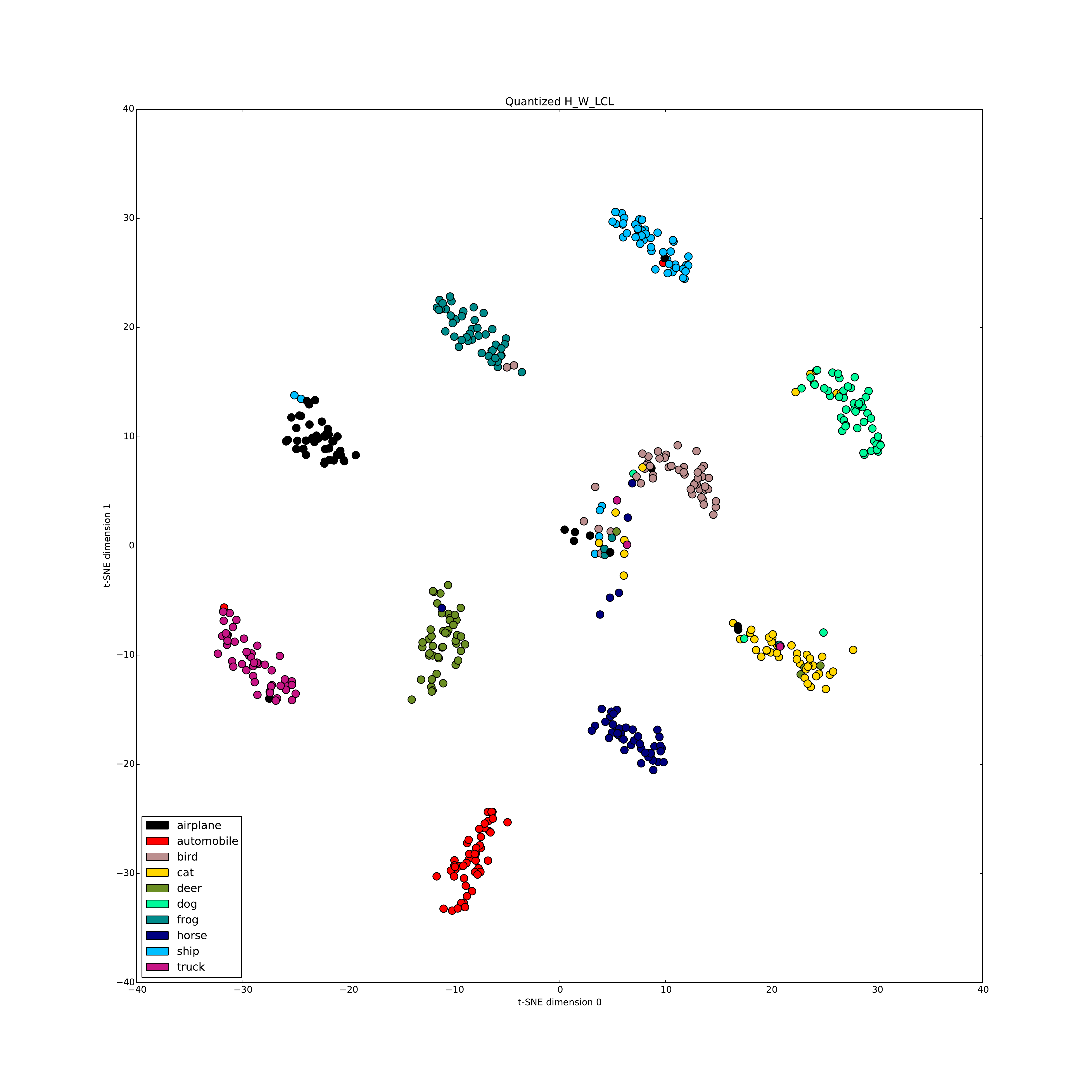}
		\caption{H + W2 + LCL + P + Q}
		\label{fig:tsne_c10_dense_H_W_LCL_Q}
	\end{subfigure}

	\caption{ tSNE visualization of last layer in densenet \cite{huang2016densely} for 50 random test samples from each class of Cifar 10 for various regularizations, here the notation in figures correpond to H = Hinge loss, W or W2 = L2 weight regularization, LCL = data dependent regularizer applied on last layer only, P= pruning applied, Q= quantization applied. We observe that in both the cases, the figures (\ref{fig:tsne_c10_dense_H_W_LCL}) and (\ref{fig:tsne_c10_dense_H_W_LCL_Q}) have better clustering than figures (\ref{fig:tsne_c10_dense_H_W}) and (\ref{fig:tsne_c10_dense_H_W_Q})}
	\label{fig:tSNE_visualization}
\end{figure*}

\begin{table*}[htbp]
	\centering
	\caption{Accuracies for various methods for 1 hidden layer FNN}
	\scalebox{0.7}{
		
		\begin{tabular}{|l|r|r|r|r|r|r|r|r|r|r|r|r|}
			\hline
			& \multicolumn{6}{|c|}{Unpruned}                  & \multicolumn{6}{|c|}{Pruned} \\
			& \multicolumn{1}{|l|}{H} & \multicolumn{1}{|l|}{H+W2} & \multicolumn{1}{|l|}{H+W1} & \multicolumn{1}{|l|}{H+LCA} & \multicolumn{1}{|l|}{H+W2+LCA} & \multicolumn{1}{|l|}{H+W1+LCA} & \multicolumn{1}{|l|}{H} & \multicolumn{1}{|l|}{H+W2} & \multicolumn{1}{|l|}{H+W1} & \multicolumn{1}{|l|}{H+LCA} & \multicolumn{1}{|l|}{H+W2+LCA} & \multicolumn{1}{|l|}{H+W1+LCA} \\
			\hline
			a9a   & 0.826 & 0.848 & 0.849 & 0.848 & 0.848 & \textbf{0.849} & 0.818 & 0.842 & 0.840 & 0.839 & \textbf{0.842} & 0.840 \\
			acoustic & \textbf{0.781} & \textbf{0.781} & \textbf{0.781} & 0.778 & 0.773 & \textbf{0.781} & \textbf{0.779} & \textbf{0.779} & \textbf{0.779} & 0.778 & 0.769 & \textbf{0.779} \\
			connect-4 & 0.815 & \textbf{0.820} & 0.819 & 0.812 & 0.813 & 0.819 & 0.809 & \textbf{0.810} & \textbf{0.810} & 0.809 & 0.805 & \textbf{0.810} \\
			dna   & 0.851 & 0.941 & \textbf{0.954} & 0.938 & 0.941 & 0.953 & 0.845 & 0.938 & \textbf{0.950} & 0.930 & 0.938 & 0.944 \\
			ijcnn & 0.968 & 0.964 & \textbf{0.974} & 0.965 & 0.964 & \textbf{0.974} & 0.962 & 0.955 & \textbf{0.967} & 0.956 & 0.955 & \textbf{0.967} \\
			mnist & \textbf{0.968} & \textbf{0.968} & 0.938 & 0.947 & 0.940 & 0.933 & \textbf{0.959} & \textbf{0.959} & 0.929 & 0.940 & 0.937 & 0.930 \\
			protein & 0.617 & 0.676 & \textbf{0.685} & 0.667 & 0.676 & \textbf{0.685} & 0.614 & 0.668 & \textbf{0.677} & 0.658 & 0.668 & \textbf{0.677} \\
			seismic & 0.737 & 0.740 & \textbf{0.741} & 0.738 & 0.740 & \textbf{0.741} & 0.729 & 0.736 & \textbf{0.741} & 0.738 & 0.736 & \textbf{0.741} \\
			w8a   & \textbf{0.988} & \textbf{0.988} & \textbf{0.988} & 0.984 & 0.982 & \textbf{0.988} & 0.979 & \textbf{0.981} & 0.979 & 0.974 & 0.972 & 0.979 \\
			webspam uni & \textbf{0.985} & \textbf{0.985} & \textbf{0.985} & 0.984 & 0.971 & \textbf{0.985} & \textbf{0.978} & \textbf{0.978} & \textbf{0.978} & 0.975 & 0.963 & \textbf{0.978} \\
			\hline
		\end{tabular}%
	}
	\label{tab:acc_NN1_pr}%
\end{table*}%

% Table generated by Excel2LaTeX from sheet 'acc_all_NN1_pr'
\begin{table*}[htbp]
	\centering
	\caption{Accuracies for various methods for 2 hidden layer FNN}
	\scalebox{0.7}{
		\begin{tabular}{|l|r|r|r|r|r|r|r|r|r|r|r|r|}
			\hline
			& \multicolumn{6}{|c|}{Unpruned}                  & \multicolumn{6}{|c|}{Pruned} \\
			& \multicolumn{1}{|l|}{H} & \multicolumn{1}{|l|}{H+W2} & \multicolumn{1}{|l|}{H+W1} & \multicolumn{1}{|l|}{H+LCA} & \multicolumn{1}{|l|}{H+W2+LCA} & \multicolumn{1}{|l|}{H+W1+LCA} & \multicolumn{1}{|l|}{H} & \multicolumn{1}{|l|}{H+W2} & \multicolumn{1}{|l|}{H+W1} & \multicolumn{1}{|l|}{H+LCA} & \multicolumn{1}{|l|}{H+W2+LCA} & \multicolumn{1}{|l|}{H+W1+LCA} \\
			\hline
			a9a   & 0.831 & \textbf{0.849} & 0.845 & 0.843 & 0.847 & 0.841 & 0.827 & \textbf{0.841} & 0.840 & 0.834 & 0.839 & 0.831 \\
			acoustic & 0.777 & 0.775 & \textbf{0.779} & 0.776 & 0.775 & 0.777 & \textbf{0.777} & 0.766 & 0.775 & 0.771 & 0.766 & 0.767 \\
			connect-4 & 0.804 & 0.817 & 0.816 & 0.816 & 0.821 & \textbf{0.824} & 0.798 & 0.813 & 0.811 & 0.808 & 0.815 & \textbf{0.823} \\
			dna   & 0.812 & 0.938 & \textbf{0.957} & 0.906 & 0.938 & 0.895 & 0.803 & 0.930 & \textbf{0.954} & 0.898 & 0.930 & 0.886 \\
			ijcnn & \textbf{0.982} & 0.980 & 0.979 & 0.972 & 0.980 & 0.979 & \textbf{0.977} & 0.972 & 0.973 & 0.967 & 0.972 & 0.974 \\
			mnist & 0.953 & 0.957 & 0.958 & 0.953 & \textbf{0.959} & 0.943 & 0.948 & 0.955 & \textbf{0.957} & 0.944 & 0.952 & 0.939 \\
			protein & 0.596 & 0.664 & \textbf{0.670} & 0.605 & 0.664 & 0.604 & 0.587 & 0.658 & \textbf{0.670} & 0.599 & 0.658 & 0.597 \\
			seismic & 0.744 & 0.743 & \textbf{0.746} & 0.725 & 0.738 & 0.738 & 0.743 & 0.740 & \textbf{0.746} & 0.725 & 0.733 & 0.731 \\
			w8a   & 0.986 & 0.985 & 0.986 & 0.972 & 0.975 & \textbf{0.987} & \textbf{0.978} & 0.976 & \textbf{0.978} & 0.970 & 0.970 & 0.977 \\
			webspam uni & \textbf{0.986} & 0.983 & \textbf{0.986} & 0.969 & 0.983 & 0.981 & \textbf{0.979} & 0.978 & \textbf{0.979} & 0.965 & 0.978 & 0.978 \\
			\hline
		\end{tabular}%
	}
	\label{tab:acc_NN2_pr}%
\end{table*}%

% Table generated by Excel2LaTeX from sheet 'acc_all_NN1_pr'
\begin{table*}[htbp]
	\centering
	\caption{Accuracies for various methods for 3 hidden layer FNN}
	\scalebox{0.7}{
		\begin{tabular}{|l|r|r|r|r|r|r|r|r|r|r|r|r|}
			\hline
			& \multicolumn{6}{|c|}{Unpruned}                  & \multicolumn{6}{|c|}{Pruned} \\
			& \multicolumn{1}{|l|}{H} & \multicolumn{1}{|l|}{H+W2} & \multicolumn{1}{|l|}{H+W1} & \multicolumn{1}{|l|}{H+LCA} & \multicolumn{1}{|l|}{H+W2+LCA} & \multicolumn{1}{|l|}{H+W1+LCA} & \multicolumn{1}{|l|}{H} & \multicolumn{1}{|l|}{H+W2} & \multicolumn{1}{|l|}{H+W1} & \multicolumn{1}{|l|}{H+LCA} & \multicolumn{1}{|l|}{H+W2+LCA} & \multicolumn{1}{|l|}{H+W1+LCA} \\
			\hline
			a9a   & 0.832 & \textbf{0.847} & 0.845 & 0.845 & \textbf{0.847} & 0.840 & 0.822 & 0.838 & 0.836 & \textbf{0.840} & 0.838 & 0.831 \\
			acoustic & 0.779 & 0.775 & \textbf{0.779} & 0.775 & 0.775 & 0.771 & \textbf{0.777} & 0.772 & \textbf{0.777} & 0.769 & 0.772 & 0.767 \\
			connect-4 & 0.815 & \textbf{0.820} & 0.816 & 0.816 & 0.817 & 0.813 & 0.805 & 0.811 & \textbf{0.816} & 0.808 & 0.808 & 0.808 \\
			dna   & 0.761 & 0.934 & \textbf{0.957} & 0.903 & 0.932 & 0.856 & 0.756 & 0.931 & \textbf{0.950} & 0.902 & 0.922 & 0.852 \\
			ijcnn & 0.980 & \textbf{0.982} & 0.981 & 0.977 & \textbf{0.982} & 0.977 & 0.972 & 0.973 & 0.974 & \textbf{0.975} & 0.973 & 0.971 \\
			mnist & 0.958 & 0.960 & \textbf{0.961} & 0.954 & 0.955 & 0.945 & 0.955 & 0.954 & \textbf{0.960} & 0.951 & 0.946 & 0.944 \\
			protein & 0.621 & 0.656 & \textbf{0.675} & 0.657 & 0.668 & 0.627 & 0.614 & 0.648 & \textbf{0.668} & 0.651 & 0.662 & 0.617 \\
			seismic & 0.736 & \textbf{0.745} & 0.742 & 0.728 & 0.727 & 0.739 & 0.736 & \textbf{0.742} & 0.735 & 0.728 & 0.722 & 0.735 \\
			w8a   & 0.970 & 0.981 & 0.980 & 0.973 & 0.972 & \textbf{0.982} & 0.970 & 0.972 & 0.971 & 0.970 & 0.970 & \textbf{0.975} \\
			webspam uni & \textbf{0.979} & \textbf{0.979} & \textbf{0.979} & \textbf{0.979} & \textbf{0.979} & \textbf{0.979} & 0.973 & 0.970 & \textbf{0.973} & \textbf{0.973} & 0.970 & \textbf{0.973} \\
			\hline
		\end{tabular}%
	}
	\label{tab:acc_NN3_pr}%
\end{table*}%

\begin{figure*}[!ht]
	\centering
	\begin{subfigure}[b]{0.3\textwidth}
		\centering
		\includegraphics[scale=0.22]{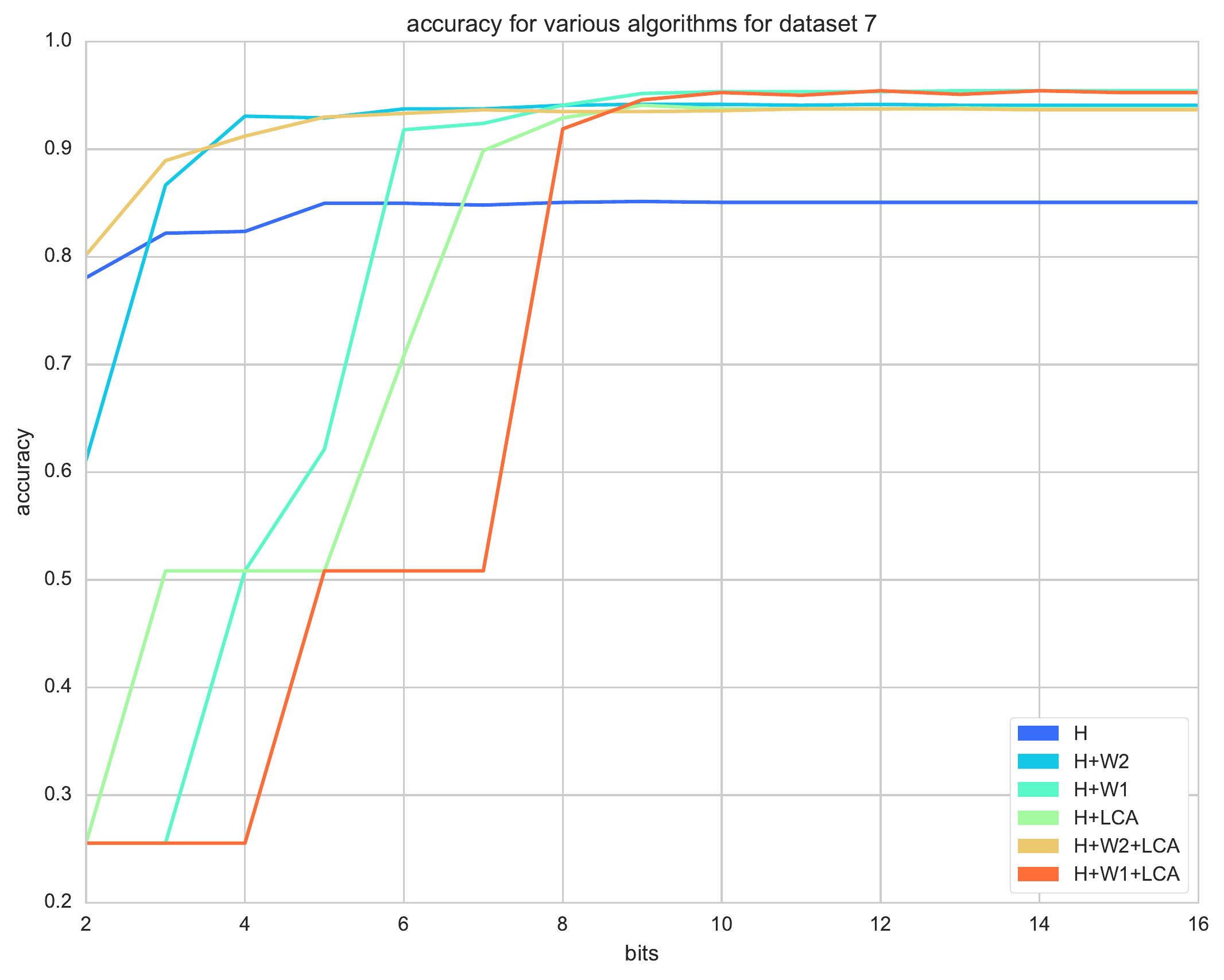}
		\caption{accuracy FNN1}
		\label{fig:acc_NN1_7}
	\end{subfigure}
	\begin{subfigure}[b]{0.3\textwidth}
		\centering
		\includegraphics[scale=0.22]{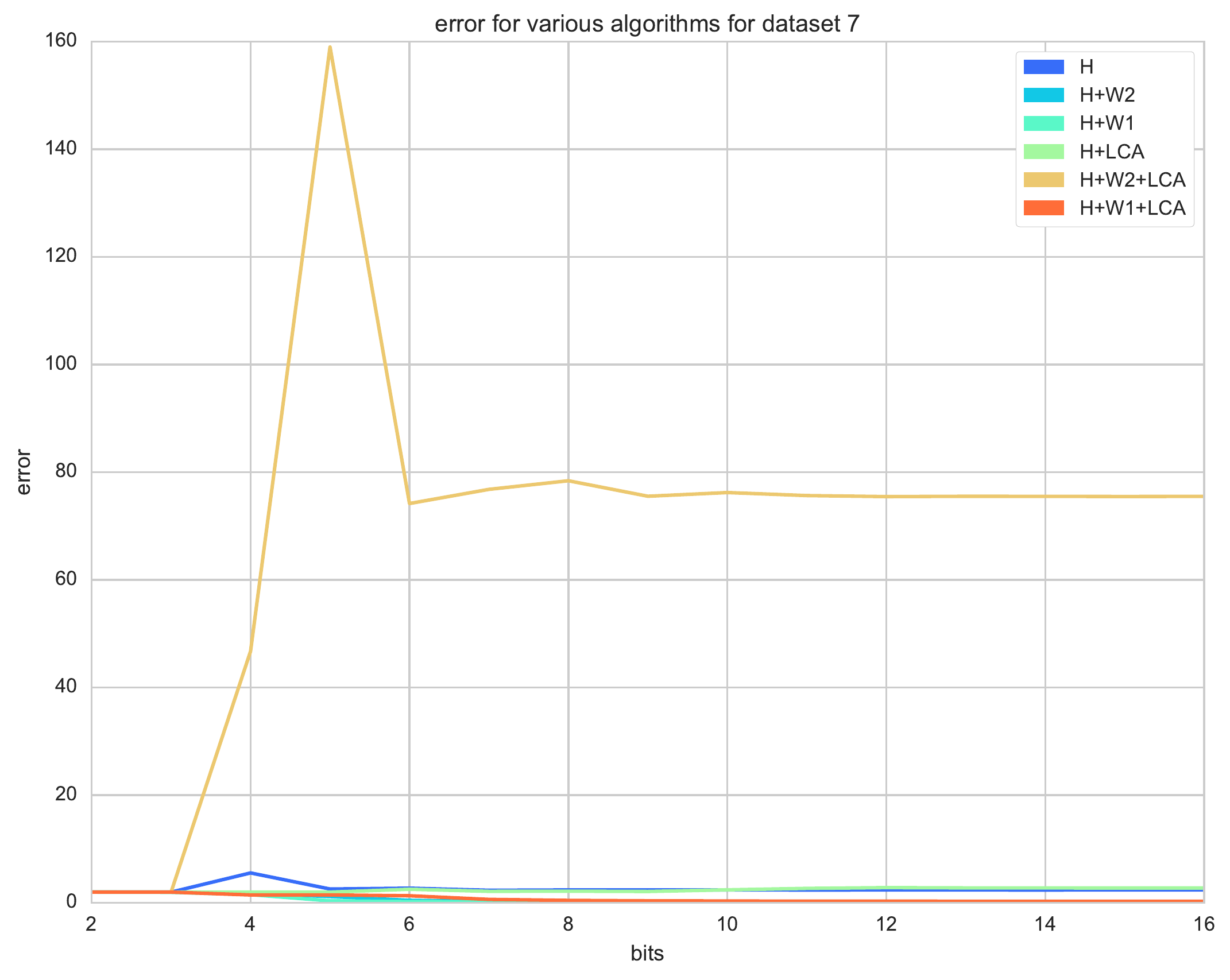}
		\caption{loss FNN1}
		\label{fig:err_NN1_7}
	\end{subfigure}
	\begin{subfigure}[b]{0.3\textwidth}
		\centering
		\includegraphics[scale=0.22]{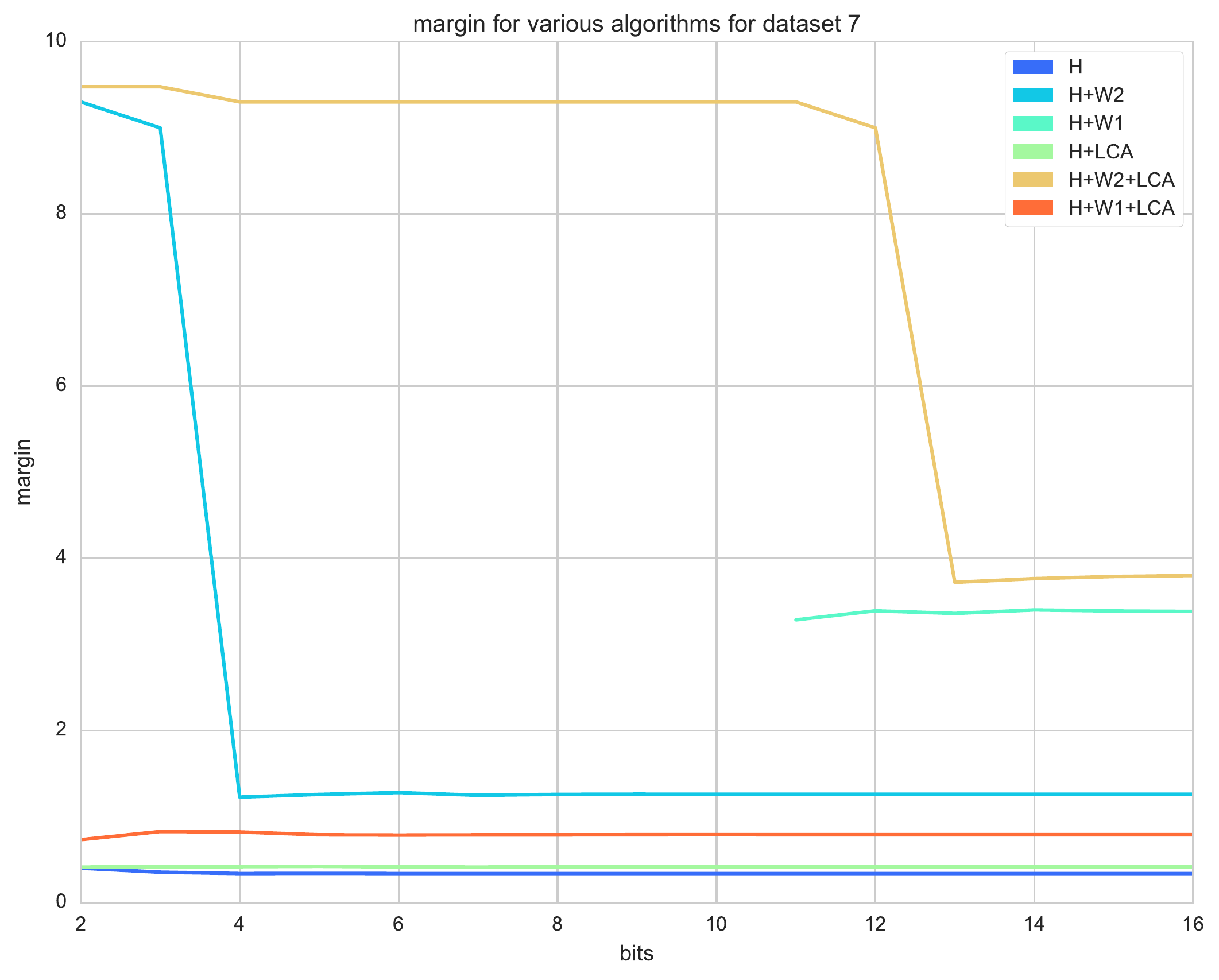}
		\caption{margin FNN1}
		\label{fig:mar_NN1_7}
	\end{subfigure}
	
	\begin{subfigure}[b]{0.3\textwidth}
		\centering
		\includegraphics[scale=0.22]{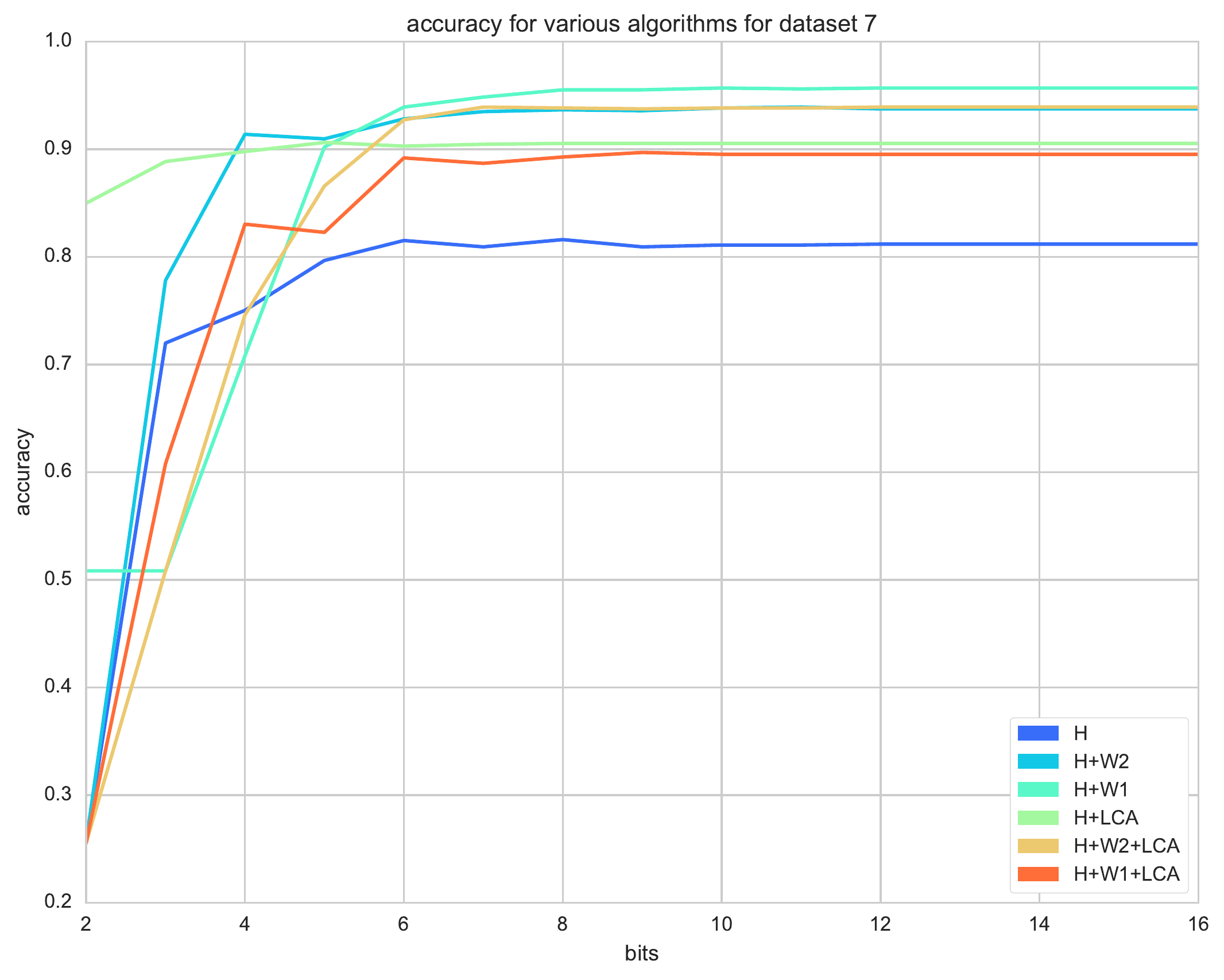}
		\caption{accuracy FNN2}
		\label{fig:acc_NN2_7}	
	\end{subfigure}
	%	add desired spacing between images, e. g. ~, \quad, \qquad, \hfill etc. 
	%	(or a blank line to force the subfigure onto a new line)
	\begin{subfigure}[b]{0.3\textwidth}
		\centering
		\includegraphics[scale=0.22]{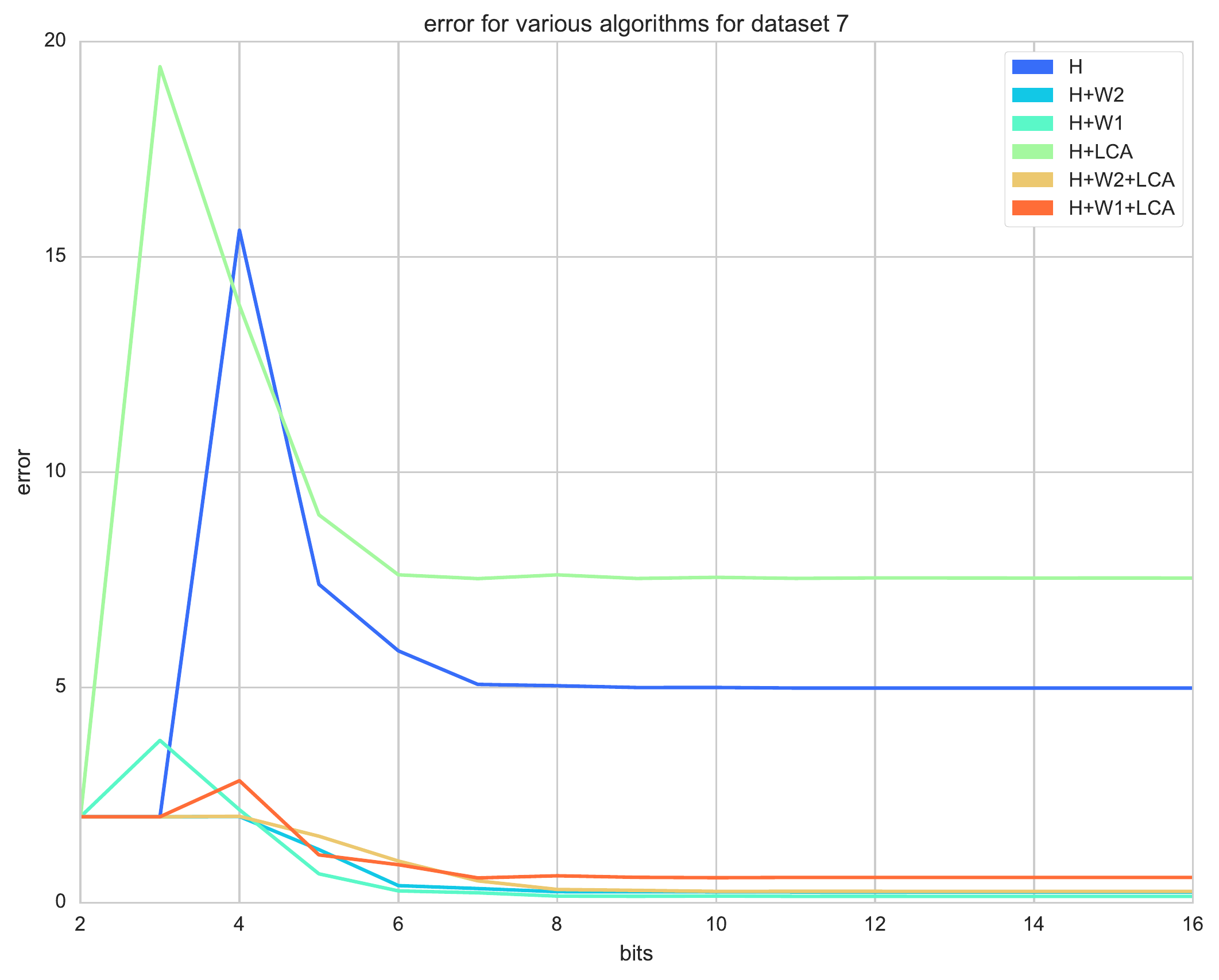}
		\caption{loss FNN2}
		\label{fig:err_NN2_7}
	\end{subfigure}
	\begin{subfigure}[b]{0.3\textwidth}
		\centering
		\includegraphics[scale=0.22]{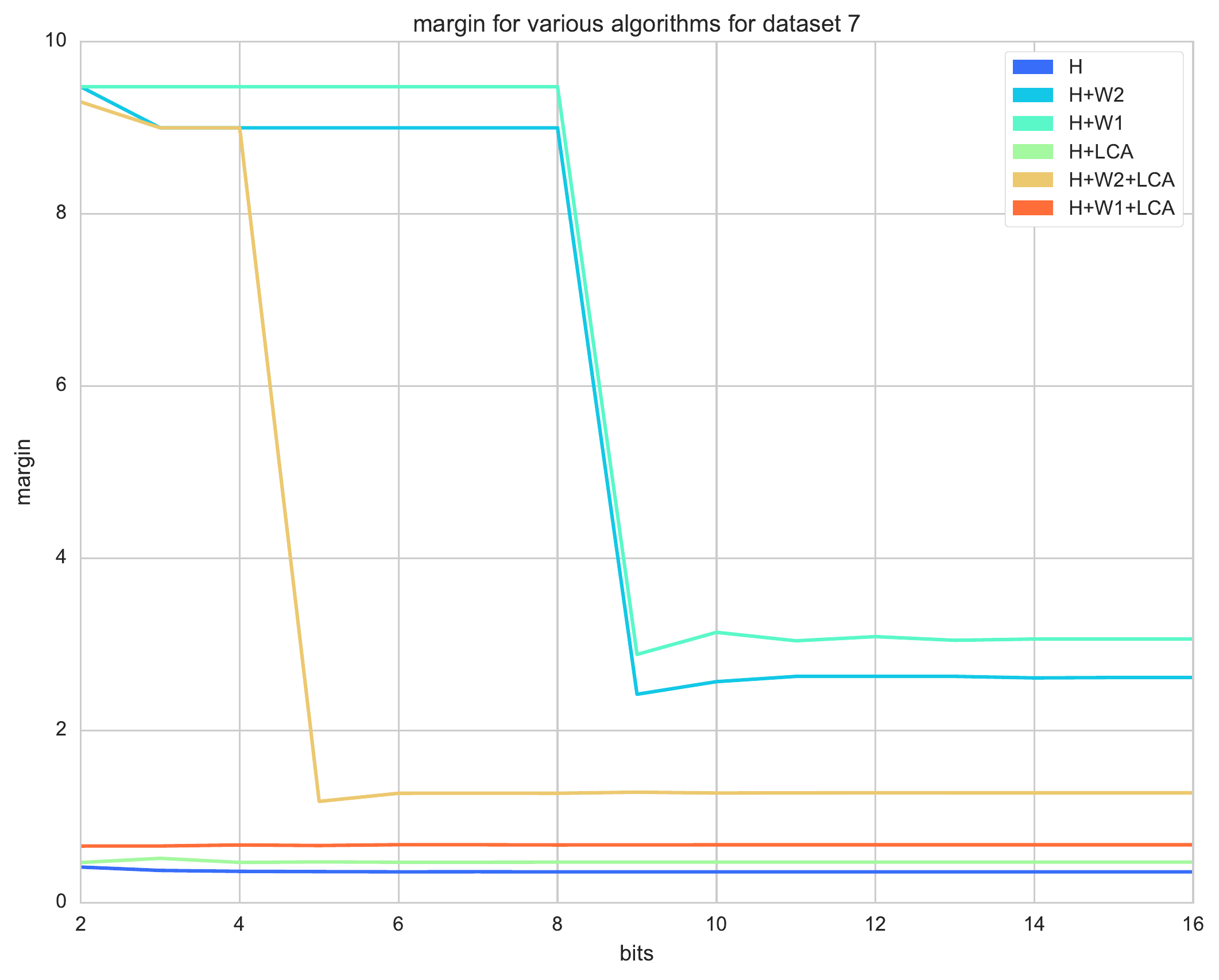}
		\caption{margin FNN2}
		\label{fig:mar_NN2_7}
	\end{subfigure}
	
	\begin{subfigure}[b]{0.3\textwidth}
		\centering
		\includegraphics[scale=0.22]{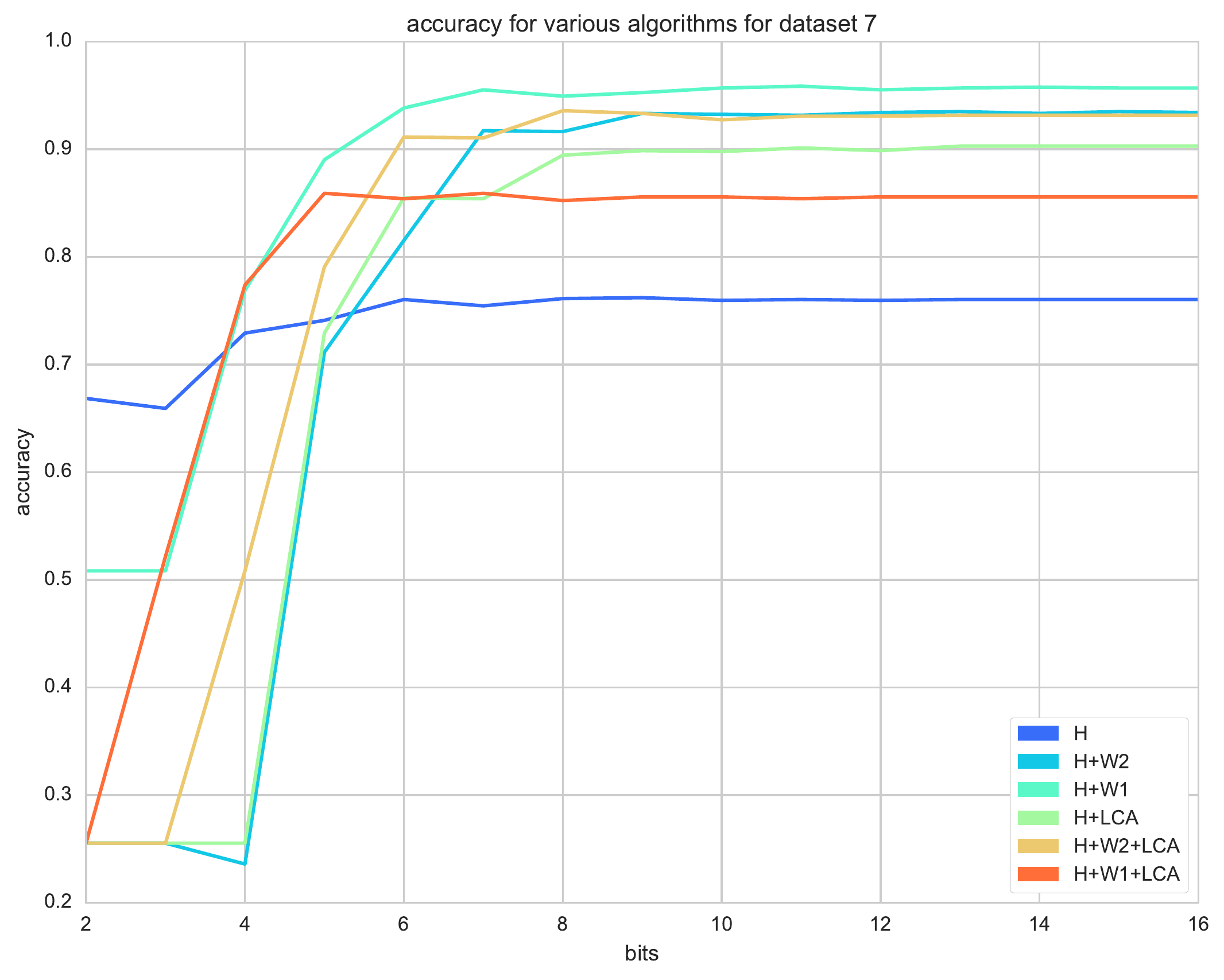}
		\caption{accuracy FNN3}
		\label{fig:acc_NN3_7}
	\end{subfigure}
	%	add desired spacing between images, e. g. ~, \quad, \qquad, \hfill etc. 
	%	(or a blank line to force the subfigure onto a new line)
	\begin{subfigure}[b]{0.3\textwidth}
		\centering
		\includegraphics[scale=0.22]{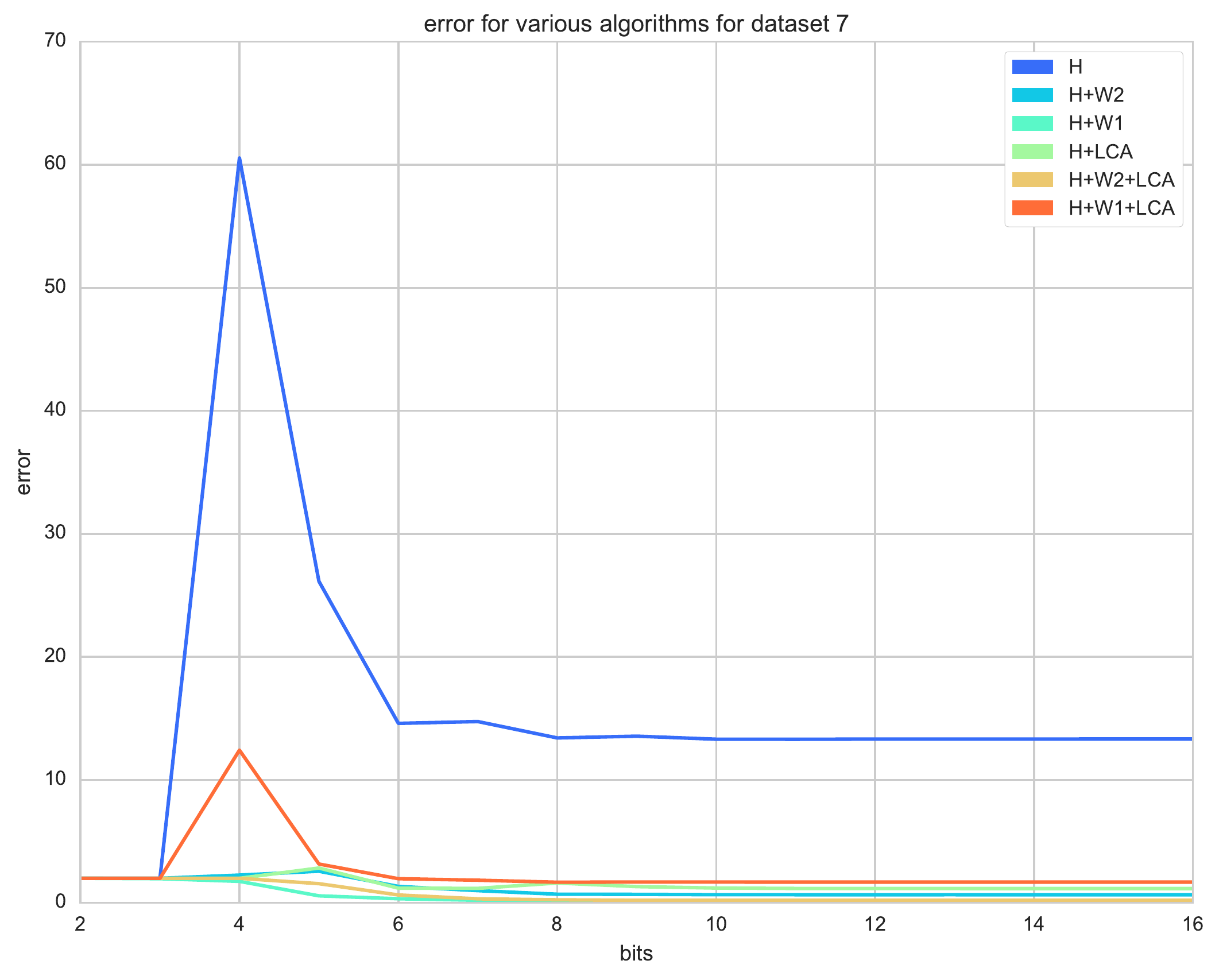}
		\caption{loss FNN3}
		\label{fig:err_NN3_7}
	\end{subfigure}
	\begin{subfigure}[b]{0.3\textwidth}
		\centering
		\includegraphics[scale=0.22]{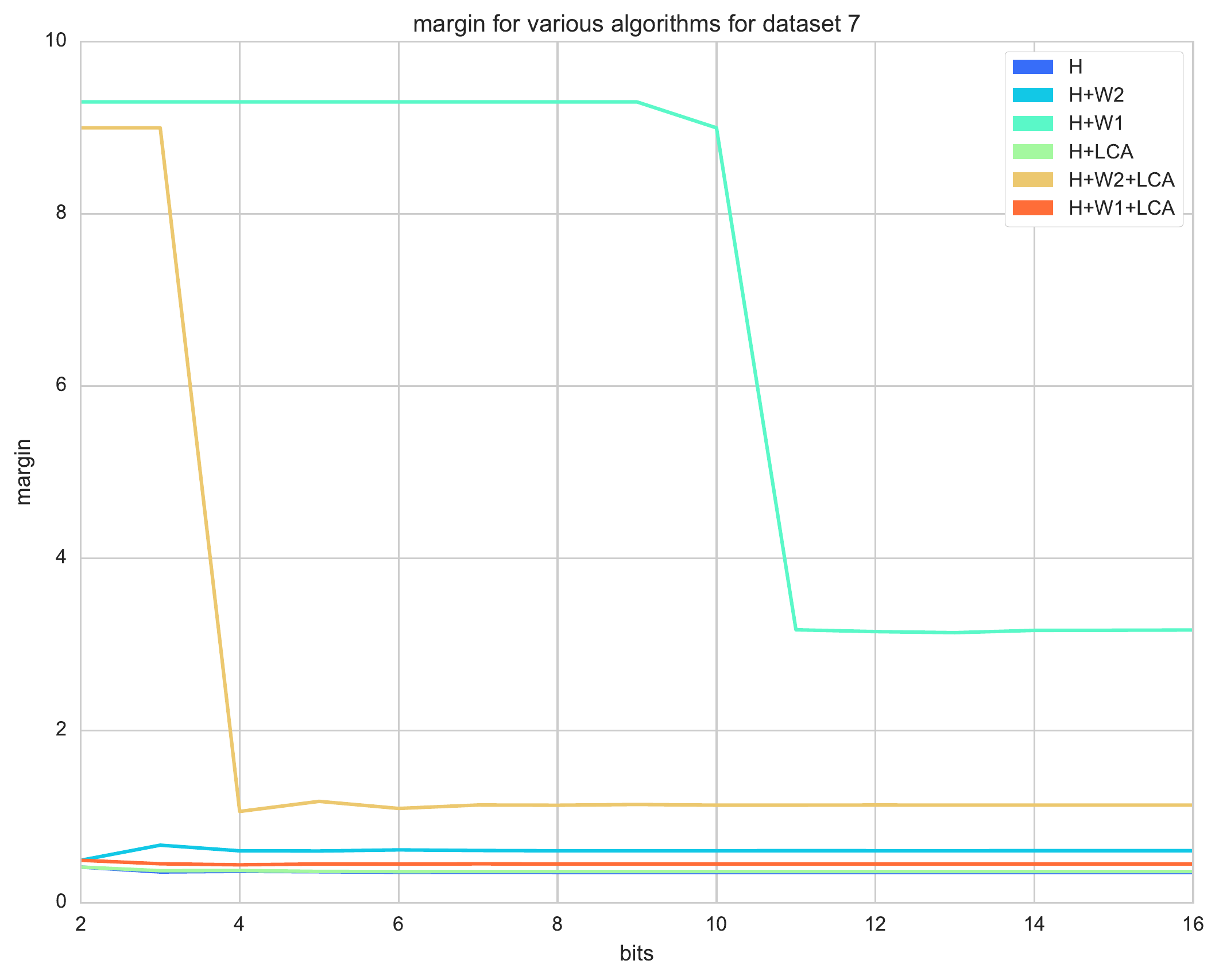}
		\caption{margin FNN3}
		\label{fig:mar_NN3_7}
	\end{subfigure}
	\caption{Effect of quantization on accuracy, margin and loss function for 1,2 and 3 hidden layer FNN for dataset 'dna'. Here we see that even on decreasing the number of total number of bits (applying brute force to determine the number of fraction bits using 1\% error tolerance), the accuracy does not significantly decrease even if total number of bits are as close to 4. In a peculiar observation, we see that for all the cases, at some value of total number of bits, the accuracy increases slightly compared to full precision. This value is different for different regularizers.}
	\label{fig:fnn_acc_margin}
\end{figure*}

Similar results were obtained for Cifar100 and MNIST datasets. Results of which can be found in supplementary section.
\subsubsection{FNNs: Datasets}
% Table generated by Excel2LaTeX from sheet 'datasets'
We use 10 datasets from LIBSVM website \cite{chang2011libsvm}, to demonstrate the effectiveness of our method when compared to other methods. The datasets vary in the number of features, classes, and training set sizes thus covering a wide variety of applications of neural networks.
\begin{table}[htbp]
	\centering
	\caption{Datasets used for FNN experiments adopted from \cite{chang2011libsvm}}
		\scalebox{0.7}{
	\begin{tabular}{|l|r|r|r|r|r|}
		\hline
		name  & \multicolumn{1}{|l|}{features} & \multicolumn{1}{|l|}{classes} & \multicolumn{1}{|l|}{train size} & \multicolumn{1}{|l|}{val size} & \multicolumn{1}{|l|}{test size} \\
		\hline
		a9a   & 122   & 2     & 26049 & 6512  & 16281 \\	
		acoustic & 50    & 3     & 63058 & 15765 & 19705 \\
		connect-4 & 126   & 3     & 40534 & 13512 & 13511 \\
		dna   & 180   & 3     & 1400  & 600   & 1186 \\
		ijcnn & 22    & 2     & 35000 & 14990 & 91701 \\
		mnist & 778   & 10    & 47999 & 12001 & 10000 \\
		protein & 357   & 3     & 14895 & 2871  & 6621 \\
		seismic & 50    & 3     & 63060 & 15763 & 19705 \\
		w8a   & 300   & 2     & 39800 & 9949  & 14951 \\
		webspam uni & 254   & 2     & 210000 & 70001 & 69999 \\
		\hline
	\end{tabular}%
}
	\label{tab:datasets_FNN}%
\end{table}%

\subsubsection{FNNs: Experiments}
In these set of experiments we show the individual effect of pruning and quantization on a wide range of regularizers prevalent in the neural network domain. We also test the efficacy of our regularizer in achieving sparsity across various neural network sizes ranging from 1 hidden layer to 3 hidden layers. The number neurons in each layer was set to 50.\\ 
Tables (\ref{tab:acc_NN1_pr}-\ref{tab:acc_NN3_pr}) shows the accuracy obtained for the datasets in case of unpruned and pruned network. We vary the number of hidden layers from 1 to 3 and evaluate the test set accuracies. We find that for 1 hidden layer FNN, $L_1$ weight regularization and $L_1$ regularization with data dependent term have the highest accuracies for 7 out of 10 datasets, whereas for pruned network $L_1$ regularization has the best performance. Similar observations can be made about networks with two and three hidden layers, where $L_1$ regularization has the best performances in terms of accuracies. \\
Tables (\ref{tab:comp_NN1_pr})-(\ref{tab:comp_NN3_pr}) demonstrates the compression ratio $r$ for individual networks. We observe that the regularizers with data dependent term outperforms in 9 out of 10 for network with 1 hidden layer, 7 out of 10 in networks with two hidden layers and 8 out of 10 in networks with 3 hidden layers. The compressions ranges from 1.0 to 5063 with just pruning.\\

% Table generated by Excel2LaTeX from sheet 'acc_all_NN1_pr'

Tables (\ref{tab:comp_NN1_pr})-(\ref{tab:comp_NN3_pr}) exhibits the compression ratio achieved by various regularizers. Here the compression ratio is defined as  $r=\frac{\text{total number of non-zero weights before pruning}}{\text{total number of non-zero weights after pruning}}$.
% Table generated by Excel2LaTeX from sheet 'com_rat_NN1'
\begin{table}[htbp]
	\centering
	\caption{Compression ratios for various methods for 1 hidden layer FNN}
	\scalebox{0.6}{
	\begin{tabular}{|l|r|r|r|r|r|r|}
		\hline
		& \multicolumn{1}{|l|}{H} & \multicolumn{1}{|l|}{H+W2} & \multicolumn{1}{|l|}{H+W1} & \multicolumn{1}{|l|}{H+LCA} & \multicolumn{1}{|l|}{H+W2+LCA} & \multicolumn{1}{|l|}{H+W1+LCA} \\
		\hline
		    a9a   & 2.1   & \textbf{133.0} & 99.2  & 2.4   & \textbf{133.0} & 99.2 \\
		acoustic & 1.2   & 1.2   & 1.2   & 1.0   & \textbf{1.5} & 1.2 \\
		connect-4 & 1.5   & 2.3   & \textbf{5.1} & 1.1   & 2.3   & \textbf{5.1} \\
		dna   & 1.8   & 167.3 & \textbf{262.9} & 19.5  & 167.3 & 85.2 \\
		ijcnn & 1.5   & \textbf{7.0} & 3.2   & 1.6   & \textbf{7.0} & 3.2 \\
		mnist & 2.3   & 2.3   & 2.3   & 1.8   & 3.0   & \textbf{6.5} \\
		protein & 1.8   & \textbf{37.1} & 35.3  & 4.6   & \textbf{37.1} & 35.3 \\
		seismic & 1.1   & \textbf{1.4} & 1.3   & 1.0   & \textbf{1.4} & 1.3 \\
		w8a   & 3.1   & 75.0  & 3.1   & 1377.5 & \textbf{1515.2} & 3.1 \\
		webspam uni & 1.3   & 1.3   & 1.3   & 2.2   & \textbf{2.8} & 1.3 \\
		\hline
	\end{tabular}%
}
	\label{tab:comp_NN1_pr}%
\end{table}%

% Table generated by Excel2LaTeX from sheet 'com_rat_NN1'
\begin{table}[htbp]
	\centering
	\caption{Compression ratios for various methods for 2 hidden layer FNN}
	\scalebox{0.6}{
	\begin{tabular}{|l|r|r|r|r|r|r|}
		\hline
		& \multicolumn{1}{|l|}{H} & \multicolumn{1}{|l|}{H+W2} & \multicolumn{1}{|l|}{H+W1} & \multicolumn{1}{|l|}{H+LCA} & \multicolumn{1}{|l|}{H+W2+LCA} & \multicolumn{1}{|l|}{H+W1+LCA} \\
		\hline
		    a9a   & 1.3   & 157.1 & 72.1  & 3.2   & \textbf{204.7} & 23.6 \\
		acoustic & 1.0   & 2.0   & 2.0   & 1.4   & 2.0   & \textbf{2.7} \\
		connect-4 & 1.4   & 1.9   & \textbf{5.0} & 1.6   & 3.8   & 4.4 \\
		dna   & 1.4   & 55.7  & \textbf{172.8} & 4.0   & 55.7  & 2.7 \\
		ijcnn & 1.6   & 7.1   & \textbf{8.1} & 2.0   & 7.1   & 6.4 \\
		mnist & 1.4   & 2.9   & 2.4   & 1.5   & 2.9   & \textbf{8.1} \\
		protein & 1.3   & \textbf{51.4} & 35.8  & 1.3   & \textbf{51.4} & 2.4 \\
		seismic & 1.1   & 2.1   & 1.5   & 1.0   & 2.1   & \textbf{6.1} \\
		w8a   & 4.4   & 4.5   & 4.4   & 2212.8 & \textbf{2528.9} & 74.7 \\
		webspam uni & 1.6   & \textbf{4.7} & 1.6   & 1.8   & \textbf{4.7} & \textbf{4.7} \\
		\hline
	\end{tabular}%
}
	\label{tab:comp_NN2_pr}%
\end{table}%

% Table generated by Excel2LaTeX from sheet 'com_rat_NN1'
\begin{table}[htbp]
	\centering
	\caption{Compression ratios for various methods for 3 hidden layer FNN}
	\scalebox{0.6}{
	\begin{tabular}{|l|r|r|r|r|r|r|}
		\hline
		& \multicolumn{1}{|l|}{H} & \multicolumn{1}{|l|}{H+W2} & \multicolumn{1}{|l|}{H+W1} & \multicolumn{1}{|l|}{H+LCA} & \multicolumn{1}{|l|}{H+W2+LCA} & \multicolumn{1}{|l|}{H+W1+LCA} \\
		\hline
		    a9a   & 1.4   & \textbf{354.7} & 283.8 & 2.2   & \textbf{354.7} & 36.0 \\
		acoustic & 1.2   & 3.4   & 1.2   & 1.5   & 3.4   & \textbf{5.9} \\
		connect-4 & 1.5   & 2.1   & 3.1   & 1.5   & 3.4   & \textbf{8.7} \\
		dna   & 1.3   & 3.6   & \textbf{596.0} & 4.4   & 18.2  & 2.9 \\
		ijcnn & 2.1   & \textbf{13.5} & 12.5  & 1.8   & \textbf{13.5} & 9.0 \\
		mnist & 1.3   & 3.6   & 3.1   & 1.6   & 3.4   & \textbf{6.2} \\
		protein & 1.3   & 6.5   & 39.4  & 1.7   & \textbf{46.9} & 2.4 \\
		seismic & 1.0   & 4.8   & 10.0  & 1.0   & \textbf{82.1} & 9.5 \\
		w8a   & 1265.8 & 2531.5 & 2531.5 & 4050.4 & \textbf{5063.0} & 22.0 \\
		webspam uni & 1.9   & \textbf{9.4} & 1.9   & 1.7   & \textbf{9.4} & 2.9 \\
		\hline		
	\end{tabular}%
}
	\label{tab:comp_NN3_pr}%
\end{table}%

\subsection{FNN:Quantization}
Figures \ref{fig:fnn_acc_margin} shows effect of quantization on the generalization abilities of neural networks. We performed quantization on the trained network. We show the accuracy, margin computed as $\frac{2}{\|w\|^2}$ and loss for multiple regularizers as the total number of bits are varied from 16 to 2. For every value of total number of bits, the number of fraction bits were varied from 3 to 15 and the number which amounted to best test set accuracies was selected . We observe that for 1 hidden layer network, the $L_1$ regularizer with data dependent term despite having the highest accuracy to start with, is the least robust as it tapers of quickly with decrease in total number of bits, whereas, $L_2$ regularizer based on minimization of VC bound is the most robust. For other networks our proposed data dependent regularizer has comparable performances to other regularizers. One peculiar observation in the figures \ref{fig:fnn_acc_margin} is that, we observe a peak in a accuracy at a certain bit value. One possible explanation can be attributed to the fact that quantization noise may allow the network to reach a better minima thus achieving higher accuracies than their full precision counterpart.

% add tables
\section{Conclusion and Discussion}\label{sec:conclusion}
This paper attempts to extend the ideas of minimal complexity machines \cite{jayadeva2015learning} and learn the weights of a neural network by minimizing the empirical error and an upper bound on the VC dimension. However, an added advantage of using such bound, is in terms of reduction in model complexity. We observe that pruning and then quantizing the models helps to achieve comparable or better sparsity in terms of weights and allows for better generalization abilities.\\

We proposed a theoretical framework to reduce the model complexity of neural networks and then ran multiple experiments on various benchmark datasets. These benchmarks offer a diversity in terms of the number of samples and number of features. The results incontrovertibly demonstrate that the our data dependent regularizer generalize better than conventional CNNs and FNNs. 

The approach presented in the paper is generic, and can be adapted to many other settings and architectures. In our experiments we use a global hyperparameter for data dependent term, which can be further improved by using multiple hyperparameters for individual layers.

{\small
	\bibliographystyle{ieee}
	\bibliography{example_paper}
}

\end{document}